\documentclass[11pt]{article}

\usepackage[utf8]{inputenc}
\usepackage[T1]{fontenc}
\usepackage{lmodern}
\usepackage{microtype}
\usepackage{ragged2e}
\usepackage[margin=1in]{geometry}

\usepackage{amsmath,amssymb,amsthm}
\usepackage{mathtools}
\usepackage{bm}
\usepackage{graphicx}
\usepackage{booktabs}
\usepackage{array}
\usepackage{xcolor}
\usepackage{enumitem}
\usepackage{caption}
\usepackage{subcaption}
\usepackage{tikz}
\usetikzlibrary{arrows.meta,positioning,calc}
\usepackage[most]{tcolorbox}

\definecolor{ramblue}{HTML}{1F77B4}
\definecolor{nvmorange}{HTML}{E66100}
\definecolor{cloudteal}{HTML}{2E7D6B}
\definecolor{rentred}{HTML}{B2182B}
\definecolor{inkgray}{HTML}{555555}

\newtcolorbox{econbox}[1][]{%
  enhanced,breakable,colback=black!2,colframe=black!65,boxrule=0.5pt,arc=1.5pt,
  left=8pt,right=8pt,top=6pt,bottom=6pt,
  fonttitle=\bfseries\small,coltitle=white,colbacktitle=black!72,
  title={#1},toptitle=2.5pt,bottomtitle=2.5pt}

\usepackage[numbers,sort&compress]{natbib}

\usepackage[colorlinks=true,linkcolor=blue!55!black,citecolor=green!45!black,%
            urlcolor=blue!55!black,breaklinks=true]{hyperref}
\usepackage{xurl}   
\usepackage{cleveref}
\hypersetup{pdftitle={Memory as a Wasting Asset: Pricing Flash Endurance for Embodied Agents, and the Limits of Doing So},pdfauthor={Josef Chen}}

\graphicspath{{figures/}}

\newcommand{\runid}[1]{}

\newcommand{\vbe}{v^{\star}_{\mathrm{BE}}}
\newcommand{\vdc}{v^{\star}_{\mathrm{DC}}}
\newcommand{\etamark}{\hat{m}}
\newcommand{\etasim}{\eta_{\mathrm{sim}}}

\theoremstyle{plain}
\newtheorem{proposition}{Proposition}
\newtheorem{corollary}{Corollary}
\theoremstyle{definition}
\newtheorem{assumption}{Assumption}
\newtheorem{definition}{Definition}

\crefname{assumption}{Assumption}{Assumptions}
\Crefname{assumption}{Assumption}{Assumptions}
\crefname{proposition}{Proposition}{Propositions}
\Crefname{proposition}{Proposition}{Propositions}
\crefname{corollary}{Corollary}{Corollaries}
\Crefname{corollary}{Corollary}{Corollaries}
\crefname{definition}{Definition}{Definitions}
\Crefname{definition}{Definition}{Definitions}

\newcommand{\Eend}{E_{\mathrm{end}}}
\newcommand{\cwear}{c_{\mathrm{wear}}}
\newcommand{\wbar}{\bar w}
\newcommand{\Vfast}{V^{\mathrm{fast}}}
\newcommand{\Vslow}{V^{\mathrm{slow}}}
\newcommand{\Ind}{\mathbf{1}}
\newcommand{\E}{\mathbb{E}}

\DeclareMathOperator{\sign}{sign}
\DeclareMathOperator{\Cov}{Cov}

\title{\bfseries Memory as a Wasting Asset:\\[3pt]
       Pricing Flash Endurance for Embodied Agents, and the Limits of Doing So}

\author{Josef Chen\\[2pt]
\small KAIKAKU\\
\small \texttt{josef@kaikaku.ai}}

\date{June 2026}

\begin{document}
\maketitle

\begin{abstract}
A robot's flash endurance is a non-renewable stock: every persisted write spends
one of a few thousand program/erase cycles and never refills, yet no fielded robot
memory system prices which memories are worth an erase cycle. We treat embodied
memory as depreciating capital and price that stock with a single endurance shadow
price $\eta$, which makes cost-minimizing placement across a RAM / on-board NVM /
cloud hierarchy a threshold in a wear-augmented per-byte index. The index is
cost-optimal whatever the sign of the value--write association $\chi$; only when
$\chi>0$ does the optimum turn non-monotone, sending a robot's most valuable
memories off its flash.

The pivot is thus empirical, and we measure $\chi$ on real robot logs at a
pre-specified gate: its sign is a property of the deployment regime---positive on
recurrent long-horizon manipulation ($\hat\chi=+1.0\times10^{-3}$, replicated at
full power), null on a shorter-horizon suite, and negative on non-recurrent
teleoperation. Two boundaries scope the result. The endurance budget is dormant on
premium $3{,}000$-P/E TLC at datasheet prices and binding on the commodity QLC/eMMC
($\sim\!1{,}000$ P/E) that cheaper edge robots run. And where it binds, a learned
wear-aware controller only ties price-based routing on task value, because realized
value is tier-invariant across RAM, NVM, and cloud: the rent governs device
lifetime and cost, not task performance. Whether wear-aware placement improves task
value remains open---$\chi$ is measured against a value proxy, and the non-monotone
optimum, while proven, is not yet observed in data.
\end{abstract}

\section{Introduction}\label{sec:intro}
A robot ships with a finite quantity of flash. Every block of its on-board NAND
tolerates a fixed number of program/erase cycles (roughly $3{,}000$ for the
TLC parts that dominate edge
platforms~\citep{kioxia_tbwpe,wd_industrialflash}), after which it wears out
and is gone. On-board memory is therefore not free storage: each persisted write
spends a fraction of a stock that does not refill, so the right object to reason
about is \emph{memory as a depreciating capital asset} carrying a per-period
user cost~\citep{halljorgenson1967tax}, not a scratchpad of unlimited capacity
(\cref{fig:hero}).
No fielded embodied-memory system prices this. They decide \emph{what} to keep;
none decide which kept memory is worth an erase cycle, in which physical tier it
should sit, or what spending the cycle costs in joules and device lifetime.

\begin{figure}[t]
  \centering
  \includegraphics[width=\linewidth]{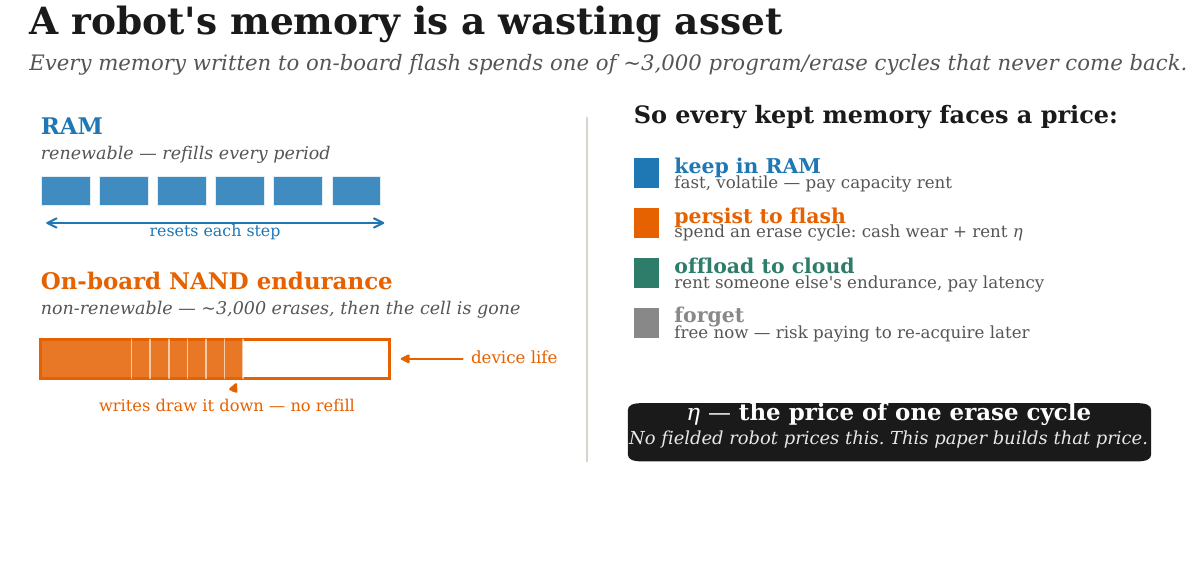}
  \caption{\textbf{The core idea in one picture.} On-board NAND endurance is a
  non-renewable stock: each persisted write spends one of a few thousand
  program/erase cycles and is gone, whereas RAM capacity refills every period.
  Every retained memory therefore faces a priced choice---keep in RAM, persist to
  flash (spending an erase cycle), offload to cloud, or forget---governed by a
  single endurance rent $\eta$. No fielded embodied-memory system prices this; we
  build that price.}
  \label{fig:hero}
\end{figure}

Any embodied memory system must answer three questions: \emph{when} to write a
memory, \emph{where} that memory should physically live, and \emph{what} it is
worth to keep it there. The first is the subject of our predecessor
\textsc{AURA}~\citep{chen2026aura}, a learned write gate; this paper supplies
the open pair. The program is a three-stage arc,
\textbf{WHEN $\rightarrow$ WHERE $\rightarrow$ WORTH} (\cref{fig:arc}): AURA gates
the write; this paper places the retained item across a RAM / on-board NVM /
cloud hierarchy; and an economic layer prices the scarce resource that placement
consumes: the erase cycle.

\begin{figure}[t]
  \centering
  \includegraphics[width=0.96\linewidth]{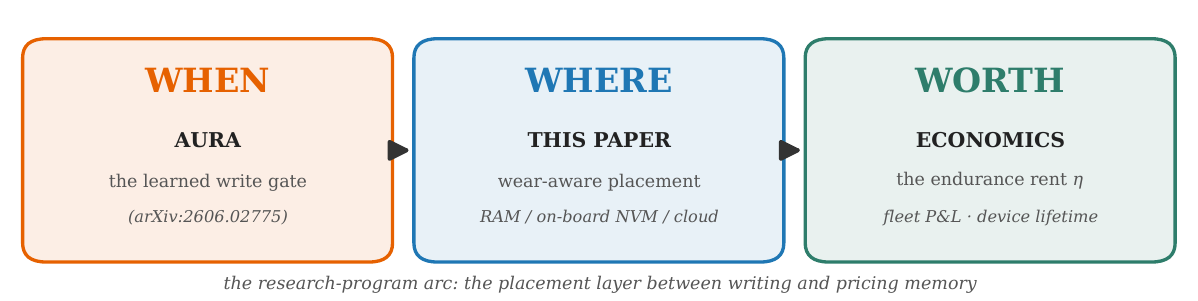}
  \caption{Research-program-arc banner. AURA decides \emph{when} to write;
  this paper decides \emph{where} memory lives; the economic layer prices
  \emph{what} it is worth (the endurance rent $\eta$).
}
  \label{fig:arc}
\end{figure}

\paragraph{WHEN.}
Embodied-memory systems are uniformly \emph{what-to-remember}
machinery~\citep{gorlo2026surprise,sridhar2025memer,wang2024karma}: they gate,
retrieve, or consolidate content to maximize task success. AURA is the author's
own \emph{when}-to-write gate on a single constant-size store. None of these
decide the physical tier an item occupies, its joule/erase cost, or its
economic worth.

\paragraph{WHERE.}
Datacenter wear-aware caching already exhibits the core qualitative behavior we
study (keeping write-heavy objects off endurance-limited
flash~\citep{flashield2019nsdi,cachesack2023tos,kangaroo2021sosp}), and we do
not claim the phenomenon. What is unoccupied is the \emph{embodied} object: in a
robot, placement is coupled simultaneously to \emph{energy}, to
\emph{task-conditioned depreciating} value, and to a \emph{cloud-offload} tier
that trades a transmit-energy-plus-latency penalty for a saved erase. We port
wear-aware admission into that regime rather than reinvent it, and stress-test
where the port survives.

\paragraph{WORTH.}
The binding constraint is an exhaustible stock, so the consumed erase cycle
carries a present-value scarcity rent $\eta$~\citep{hotelling1931economics};
memory becomes \emph{depreciating capital} with a user
cost~\citep{halljorgenson1967tax,tokeneconomics2026,marginaltoken2026}. We make
$\eta$ the operative economic object: it fixes the placement boundary, signs
how placement reacts to the 2025--26 memory-price
supercycle~\citep{trendforce2025asp}, and, because spending an erase cycle
consumes device lifetime, doubles as a fleet e-waste / embodied-carbon lever
\citep{weppe2025nandcarbon,bashir2023embodiedpitfalls}.

\paragraph{A measured, not assumed, antecedent.}
The wear-augmented index and its rent $\eta$ are the optimal policy form
\emph{however} value and write-intensity covary; $\eta$ binds whenever the
endurance stock is scarce, and scarcity is a regime, not a given---dormant on
premium TLC at datasheet prices, binding on the commodity QLC/eMMC cheaper edge
robots run (\cref{sec:results-binding}). The \emph{non-monotone} refinement
(\cref{prop:nonmono}) needs one further primitive, a positive value--write
association $\chi>0$, which we treat as a \emph{falsifiable antecedent} and measure
on real robot logs at a pre-specified \$25 gate, with a published kill criterion,
before any controller is trained (\cref{ass:assoc}). The headline empirical
finding is that $\chi$'s sign is a property of the deployment regime, not a
universal law: positive on recurrent long-horizon manipulation with a small
backbone, null on a shorter-horizon suite, and negative on non-recurrent
teleoperation. The coupling tracks long-horizon \emph{recurrence}---re-observation
of valuable scenes couples write-intensity with value, whereas value-agnostic
teleoperation churn decouples it---and it is real but small. We report it only
where a pre-specified cross-backbone agreement floor is met: a larger OpenVLA-7B
backbone places items on a near-orthogonal value axis and is \emph{uninterpretable}
against the headline rather than a disconfirmation. Full estimates, clustering, and
corrections are in \cref{sec:results-chi,sec:results-dose}.

\paragraph{Contributions.}
\begin{enumerate}[leftmargin=1.4em,itemsep=2pt]
\item \textbf{Measurement: the value--write coupling's sign is regime-dependent}
(\cref{sec:results-chi,sec:results-dose}). On real robot logs at a pre-specified
\$25 gate, $\chi$ is positive on recurrent long-horizon manipulation (LIBERO-Long,
SmolVLA-0.5B; Holm-reject, CI excluding zero), null on a shorter-horizon suite, and
negative on non-recurrent teleoperation (DROID; post-hoc), with a recurrence
dose-response that replicates at full power ($\rho=0.94$, $p<10^{-4}$). The coupling
tracks long-horizon recurrence, not a dataset. We pair it with a \emph{cross-backbone
agreement floor} (pre-specified Spearman $\ge 0.6$), below which cross-model sign
claims are uninterpretable, as our OpenVLA-7B arm shows ($\rho_s=0.05$).
\item \textbf{Theory: a wear-augmented placement index and a conditional
non-monotone optimum} (\cref{sec:model}). Cost-minimizing placement across
RAM/NVM/cloud is a threshold in a per-byte index set by a single endurance shadow
price $\eta$, optimal regardless of the sign of $\chi$. On this sign-agnostic spine,
a \emph{proven} strictly-non-monotone-in-value optimum (\cref{prop:nonmono}) holds
on the $\chi>0$ branch---with the antecedent measured, not assumed (\cref{ass:assoc}).
\item \textbf{Boundary: when the pricing layer is live, and when it is not}
(\cref{sec:results-binding,sec:results-controller}). At datasheet prices the
endurance budget is dormant on premium $3{,}000$-P/E TLC but binding on the
commodity QLC/eMMC ($\sim\!1{,}000$ P/E) cheaper edge robots run. Where it binds, a
$3.15$M-parameter learned controller is genuinely endurance-aware---it strictly
beats the naive all-NVM strategy---but only ties the strongest cost-matched baseline
on a task-value proxy. The tie is structural: with the cloud repriced at its slow
value and connectivity swept, the wear-aware advantage stays zero because LIBERO
write-intensity is nearly constant ($\mathrm{CV}(w)=0.13\%$), collapsing the index
to value-ranking; a synthetic control recovers the advantage only once
$\mathrm{CV}(w)$ is large (\cref{fig:dispersion}). On today's hardware and workloads
simple price-based routing suffices, and whether wear-awareness improves task value
is unresolved.
\item \textbf{Economics: a calibrated capital model of wasting memory}
(\cref{sec:pricestatics,sec:cscy}). Signed price comparative statics over an
oligopoly band (\cref{prop:statics}) confirm three of four predicted signs.
Re-solving across the 2025--26 NAND supercycle cuts the equilibrium rent $\etasim$
by $\approx 39\%$ while the break-even durability $\vbe\!=\!0.91$ holds fixed---the
shock hits the wear margin, not the placement boundary---and a bounded corollary
links cost-optimal forgetting to device-lifetime extension (\cref{cor:lifetime}).
\end{enumerate}

Every headline claim carries an epistemic tier, sorted in \cref{fig:ledger}: what
is \emph{proven} as theory, what is \emph{measured} on a value proxy, what is
\emph{regime-gated} by the hardware, and the one \emph{negative} result. (Each
number maps to its run and data file in the reproducibility appendix.)

\begin{figure}[t]
  \centering
  \includegraphics[width=\linewidth]{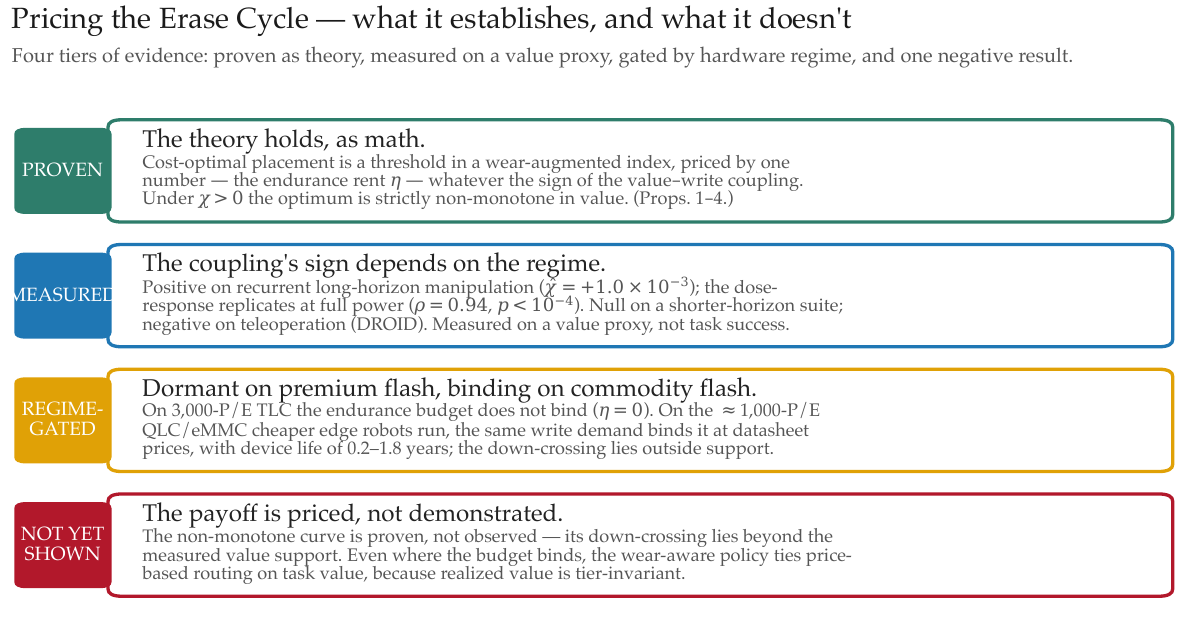}
  \caption{\textbf{Every headline claim, by epistemic tier.} What is
  \emph{proven} (the wear-augmented index and the conditional non-monotone
  optimum, as theorems); \emph{measured} on a value \emph{proxy} (the
  regime-dependent sign of $\chi$, replicated but small); \emph{regime-gated} (the
  budget does not bind on premium $3{,}000$-P/E TLC but \emph{binds at datasheet
  prices on commodity QLC/eMMC}; the down-crossing lies outside support); and a
  \emph{negative result} (even where the budget binds, the wear-aware policy ties
  price-based routing on task value---realized value is tier-invariant).}
  \label{fig:ledger}
\end{figure}

The lead empirical figure is \cref{fig:chimatrix}, the measured
backbone$\times$regime $\chi$ matrix; the model-derived wear phase diagram
(\cref{fig:phase}) and its interior down-crossing $\vdc$ are a theory illustration,
deferred to \cref{sec:model}.

\section{Related Work}\label{sec:related}
Our regime sits at the intersection of six literatures. Two of its ingredients are
old: datacenter storage already keeps high-value, write-heavy items off
endurance-limited flash, and already solves an endurance-budgeted admission
knapsack. What is new is the \emph{joint} embodied object---three-tier RAM/NVM/cloud
placement under a \emph{simultaneous} energy and \emph{non-renewable} endurance
budget, with task-conditioned depreciating value and a priced exhaustible-stock
shadow price $\eta$---which no single prior literature spans. Each subsection names
the closest prior art and what it leaves open.

\subsection{Wear-aware flash caching and storage}\label{sec:rw-wear}
The closest prior work is \textsc{Flashield}~\citep{flashield2019nsdi}: its
learned admission filter uses DRAM to keep write-amplifying objects off flash
under a write-rate cap, already exhibiting the qualitative phenomenon we
analyze---persistence is not monotone in an object's worth. What we add is its
\emph{driver and formalization}: a proven down-crossing driven by a priced
non-renewable endurance stock and value depreciation $\delta$, in a three-tier
energy-budgeted embodied loop with a cloud-offload action Flashield lacks.
\textsc{CacheSack}~\citep{cachesack2023tos} solves a per-category admission
knapsack that cuts Google datacenter flash wearout by $17.8\%$---our
placement-index skeleton at cloud scale, but with endurance as a soft cost term
rather than a hard finite stock, and no energy, depreciating value, or multi-tier
action. \textsc{Kangaroo}~\citep{kangaroo2021sosp} supplies the lifetime-bounded
write-cost-threshold admission rule we adopt as a cost-matched baseline
(\cref{tab:baselines}). Managed-Retention Memory~\citep{mrm2025} \emph{calls for}
an endurance budget plus retention-aware placement but supplies no controller,
depreciating-value model, or shadow-price theorem; DPRO~\citep{dpro2025caching}
learns per-content retention with a soft P/E cost but keys on content popularity, a
single tier, and no energy or capital layer. The learned write-avoidance
line~\citep{wang2018avoidwrites,mlwp2020date,nemo2026fast} makes steering writes
around flash standard practice, so our controller is a necessity, not a novelty
claim.

\subsection{Embodied / robot memory and VLA models}\label{sec:rw-embodied}
Embodied-memory systems decide \emph{what} to remember; none decide \emph{where}
a retained item lives or what it costs in joules and erase cycles.
Surprise-gating~\citep{gorlo2026surprise} produces a value proxy we can place under
a budget, but not a placement decision; \textsc{MemER}~\citep{sridhar2025memer}
bounds context cost by keeping $\le 8$ keyframes and lists discarding them as
future work, naming the eviction-under-budget gap we close;
\textsc{ReMEmbR}~\citep{anwar2024remembr} and \textsc{KARMA}~\citep{wang2024karma}
build and prune memory stores for recall and prompt relevance, not for a joule or
erase budget across tiers. Our own \textsc{AURA}~\citep{chen2026aura} is the launch
point---a \emph{when}-to-write gate on a single constant-size store---which leaves
the where-and-worth layer open and serves as the single-tier baseline our
controller must dominate (\cref{tab:baselines}). MemGPT's OS-style token
paging~\citep{packer2023memgpt} moves text between fast and slow stores to relieve
capacity, abstracting away the hardware-wear and dollar economics that are our
subject. The backbones we evaluate, SmolVLA~\citep{shukor2025smolvla} and
OpenVLA~\citep{kim2024openvla}, carry no persistent-memory mechanism; their
edge-deployability is what makes RAM/NVM/cloud placement economically live.

\subsection{Economic and decision-theory foundations}\label{sec:rw-econ}
The economics is assembled from mature toolkits, each applied to a new object. The
Hall--Jorgenson user cost of capital~\citep{halljorgenson1967tax,jorgenson1967theory}
supplies ``memory as a depreciating asset'' with per-period rent $=$ holding cost
$+$ depreciation $\delta$; Hotelling's exhaustible-resource
theory~\citep{hotelling1931economics} supplies the erase cycle as a unit of a
non-renewable stock carrying a scarcity rent $\eta$ (demoted to a bounded caveat,
\cref{sec:hotelling}, since the closed-form price path fails under stochastic
demand). That $\eta$ decouples the per-item program follows from Lagrangian
relaxation---Whittle's restless-bandit subsidy~\citep{whittle1988restless}, the
Gittins index~\citep{gittins1979bandit}, weakly-coupled-MDP
relaxation~\citep{adelmanmersereau2008relaxations}, and constrained-MDP
duality~\citep{altman1999cmdp}---with the new ingredient that our coupling
constraint is an \emph{intertemporal stock}, not a per-period one, which yields the
value down-crossing. Pricing agent memory as depreciating capital is itself not new
in the \emph{token}-budget setting: Token Economics~\citep{tokeneconomics2026} and
the Marginal-Token-Allocator~\citep{marginaltoken2026} give cache-as-inventory
shadow prices, but for token and context budgets with no physical endurance stock,
RAM/NVM/cloud placement, energy term, or non-monotone optimum---so our claim is
re-scoped to the \emph{physical} P/E stock. Omri et
al.~\citep{omri2026agentmemory} profile stateful agent-memory cost without a shadow
price or capital model, and inference-aware deployment
economics~\citep{sardana2024beyondchinchilla} motivates pricing write, hold, and
retrieve jointly over the horizon.

\subsection{Edge/cloud offloading, robot hardware, and cost anchors}\label{sec:rw-edge}
Robot computation offloading is a learned when-to-offload decision: Chinchali et
al.~\citep{chinchali2019offloading} solve perception offload under stochastic
networks with deep RL; we reuse that machinery for the persist-versus-offload
decision and use their policy as the ``this is just offloading'' rebuttal baseline
(it offloads compute, not a persistent store, with no endurance stock).
Neurosurgeon~\citep{kang2017neurosurgeon} and the split-computing
survey~\citep{matsubara2022splitcomputing} supply the bandwidth-and-energy cost
terms, but partition the compute graph, not memory state. We correct the cloud-tier
dollar term for concurrency using Patil's utility-aware
methodology~\citep{patil2026concurrency} (a naive per-token estimate is off by
$1/U$, a $2.5$--$24\times$ penalty), and anchor edge-decode energy/latency to our
own batch-1 measurements~\citep{chen2026decodegap}. Endurance and energy constants
are datasheet-pinned to the Jetson Thor and Orin
platforms~\citep{nvidia2025jetsonthor,nvidia2022jetsonorin}.

\subsection{Caching theory, learned policies, and RL-for-systems}\label{sec:rw-caching}
The offline optimum our hindsight solver relaxes is Belady's clairvoyant
replacement~\citep{belady1966study}, with cost-aware competitive vocabulary from
weighted-paging primal-dual analysis~\citep{bansal2012primaldual} (whose fetch cost
is renewable, not a consumable wear stock). The learned-caching line gives our
recipe and baselines: LRB~\citep{song2020lrb} regresses to a relaxed Belady
boundary and Parrot~\citep{liu2020parrot} imitates the oracle, validating the
behavior-cloning warm-start we use---except our oracle solves a knapsack-over-time
under an endurance budget. Baleen~\citep{wong2024baleen} is the nearest
write-cost-aware learned cache, but admits under a write-\emph{rate} constraint with
no non-renewable stock, depreciation, energy, or cloud tier. Our PPO-on-placement
follows RL-for-systems
precedent~\citep{jay2019aurora,mao2016deeprm,mirhoseini2017deviceplacement}; the
novelty is the resource (memory tiers with consumable endurance), not the verb.
Because learned caches underperform heuristics under abundant cache or distribution
shift~\citep{qiao2023frozenhot,lcr2025laru}, we report cost-matched baselines and
bound the worst case to a tuned heuristic.

\subsection{Memory market band and cs.CY / policy context}\label{sec:rw-market}
The 2025--26 memory supercycle hands the price-statics layer a dated, citable grid:
enterprise TLC NAND roughly \$0.06--0.22/GB across the Low/Base/High band, anchored
by TrendForce ASP tracking~\citep{trendforce2025asp}, Counterpoint server-DRAM
analysis~\citep{counterpoint2025serverdram}, and Epoch AI's component cost
model~\citep{epochai2024b200}. The cs.CY claim---endurance-aware placement extends
device life and defers fleet embodied carbon---rests on Weppe et al.'s $\approx
22$~kg~CO$_2$e/TB for 3D NAND~\citep{weppe2025nandcarbon}, within the ACT
carbon-modeling framework~\citep{gupta2022act} and AI-hardware
LCA~\citep{schneider2025lifecycle,pirson2021iotcarbon}. Its bounds: no regulation
sets a numeric P/E floor, lifetime extension can overstate real
savings~\citep{bashir2023embodiedpitfalls}, and superlinear new-hardware efficiency
can justify replacement over retention~\citep{switzer2022junkyard}. The EU
circular-economy regime (Right-to-Repair~\citep{eu2024rtr},
Ecodesign-for-Sustainable-Products~\citep{eu2024espr}) brings storage products in
scope, so endurance-aware placement complements, and is not mandated by, the policy
frontier.

\section{Model}\label{sec:model}
This section states the formal core: primitives and assumptions A1--A6, the
placement index, the conditional non-monotonicity result
(Prop.~\ref{prop:nonmono}), the RAM-survival result (Prop.~\ref{prop:ramslack}),
the Hotelling caveat with the cs.CY corollary, and the price comparative statics
P5a--d. Proofs are deferred to \cref{app:proofs}; \cref{fig:hero} shows the priced
loop and the box below states its logic in plain English.

\begin{econbox}[The economic logic in plain English]
\small A robot's flash ships with a fixed stock of erase cycles. The
cost-minimizing planner prices that stock with one number, the \emph{endurance
rent} $\eta$---what one erase cycle is worth in the best feasible plan. Each
candidate memory then faces a capital-budgeting test: persist on-board only if
being \emph{local} covers storage, energy, \emph{and} the full user cost of wear
$(\cwear+\eta)\,w_i$; else rent the cloud's endurance (paying latency) or forget it
and risk re-acquisition. When memory prices move, the budget still binds and the
rent re-clears---re-pricing the marginal memory while the persist/evict boundary
barely moves (\cref{sec:pricestatics}).
\end{econbox}

\subsection{Primitives and assumptions}\label{sec:primitives}
Discrete time $t=0,\dots,T-1$, finite horizon $T$, discount $\gamma\in(0,1)$. At
each $t$ an item stream arrives (admitted upstream by AURA; admission is
exogenous here; we price \emph{placement}). Item $i$ has type
$\theta_i=(v_i,\delta_i,\lambda_i,s_i,w_i,\kappa_i)$ drawn i.i.d.\ from a joint
law $F$.

\begin{assumption}[value / depreciation / retrieval]\label{ass:value}
Base value $v_i\ge 0$ (monetized task-success gain); geometric staleness
$\delta_i\in(0,1]$ (value at age $a$ is $v_i e^{-\delta_i a}$); Poisson retrieval
rate $\lambda_i>0$.
\end{assumption}

\begin{assumption}[size / write-intensity / recompute]\label{ass:size}
Size $s_i>0$; NVM erase ops per period if resident $w_i\ge 0$; recompute cost
$\kappa_i\ge 0$ if discarded and later needed. $v$ and $w$ are
\emph{distinct coordinates} of $\theta$: nothing in the model forces them to co-move.
\end{assumption}

\begin{assumption}[tiers]\label{ass:tiers}
$k\in\{R,N,C\}=\{$RAM, NVM, cloud$\}$. Per-byte holding rents $p_R,p_N,p_C\ge 0$;
RAM hard capacity $\sum_{i\in R_t}s_i\le C_R$ (multiplier $\mu_R\ge 0$); power cap
$\sum_i\mathrm{pow}_{i,t}\le P$ (multiplier $\mu_P\ge 0$). Access cost
$\approx 0$ (RAM), I/O energy $e_N$ (NVM), transmit energy $+$ latency
$e_C+\ell\pi$ (cloud).
\end{assumption}

\begin{assumption}[non-renewable endurance: the load-bearing asymmetry]\label{ass:endurance}
NVM residency consumes erase cycles from a \emph{fixed stock}
\[
  \sum_{t=0}^{T-1}\sum_{i\in N_t} w_i \;\le\; \Eend
  \;=\;(\text{erase cycles/block})\times(\text{blocks}).
\]
This is the \emph{only} constraint integrated over time; RAM and power are flow
constraints that reset each period. Let $\eta\ge 0$ be the single multiplier on
this budget: the \textbf{shadow price of one erase cycle}.
\end{assumption}

\begin{assumption}[value--write association: empirically testable antecedent]\label{ass:assoc}
\begin{sloppypar}
Let $\wbar(v):=\E[w\mid v]$ and define the \textbf{association coefficient}
$\chi:=\tfrac{d}{dv}\wbar(v)$ (globally $\sign\chi=\sign\Cov_F(w,v)$ when $\wbar$
monotone). \cref{ass:assoc} is \emph{not assumed}: it is estimated on real
robot logs at a pre-specified \$25 go/no-go gate with a published kill: if
$\chi\le 0$ the non-monotone headline (\cref{prop:nonmono}) is withdrawn for the
monotone index (\cref{prop:monotone}).
\end{sloppypar}
\end{assumption}

\begin{assumption}[recompute sub-linearity]\label{ass:recompute}
$\bar\kappa(v):=\E[\kappa\mid v]$ has $\bar\kappa(v)/v$ non-increasing:
high-value items are not proportionally more expensive to regenerate. Testable
at the gate.
\end{assumption}

\paragraph{The program.}
The agent chooses $x_{i,t}\in\{R,N,C,\varnothing\}$, earns
$\rho_{i,k}(t)=$ (depreciated value) $-$ (access cost) $-$ (recompute if
discarded and needed), and solves
\begin{equation}\label{eq:program}
\max_{\{x_{i,t}\}}\ \E\!\Big[\sum_{t}\gamma^t\sum_i\big(\rho_{i,x_{i,t}}(t)
  -p_{x_{i,t}}s_i -\cwear w_i\,\Ind[x_{i,t}=N]-\text{migr}\big)\Big]
\ \text{s.t.}\
\begin{cases}
\sum_{i\in R_t}s_i\le C_R & \forall t,\\
\sum_i\mathrm{pow}_{i,t}\le P & \forall t,\\
\sum_t\sum_{i\in N_t}w_i\le \Eend.
\end{cases}
\end{equation}
The Lagrangian carries flow multipliers $\mu_R(t),\mu_P(t)$ and the single stock
multiplier $\eta$.

\subsection{The placement index and the wear-augmented index}\label{sec:index}
\paragraph{Renewable limit $\Eend\to\infty\Rightarrow\eta=0$.}
Given $(\mu_R,\mu_P)$ the per-period problem decouples across items. With
discounted locality return
\begin{equation}\label{eq:Vlocal}
  V_i=\sum_{a\ge 0}\gamma^a\lambda_i e^{-\delta_i a}v_i
     =\frac{\lambda_i v_i}{1-\gamma e^{-\delta_i}},
\end{equation}
place $i$ in the fastest tier whose marginal per-byte rent its index clears,
the cutoff statistic being the per-byte index
\begin{equation}\label{eq:index}
  \boxed{\;I_i=\frac{\lambda_i\,(v_i+\kappa_i)}{s_i\,(1-\gamma e^{-\delta_i})}\;}
\end{equation}
strictly increasing in $v_i$ (higher value $\Rightarrow$ faster tier: monotone),
the ``obvious'' Belady/Gittins-with-rent regime, \emph{not} the contribution.

\begin{proposition}[monotone renewable benchmark]\label{prop:monotone}
With $\eta=0$ and the energy ordering $e_N<e_C+\ell\pi$ (on-board access cheaper
than cloud), $\Ind[x_i=N]$ is weakly increasing in $v_i$, ceteris paribus.
\end{proposition}
\noindent\emph{(Proof in \cref{app:proofs}. Tested: endurance-relaxed arm,
expect monotone fit.)}

\paragraph{Binding endurance $\eta>0$: the wear-augmented index.}
Now $\Eend$ binds. NVM residency carries an \emph{extra} per-period cost
$(\cwear+\eta)w_i$: cash wear \emph{plus} the scarcity value of the consumed
erase cycle. With the RAM multiplier $\mu_R$ carried explicitly, the four
tier-specific net returns are
\begin{align}
\Pi_i^R&=\Vfast_i-(\mu_R+p_R)s_i,
  & \Pi_i^N&=\Vfast_i-p_N s_i-(\cwear+\eta)w_i,\label{eq:PiRN}\\
\Pi_i^C&=\Vslow_i-p_C s_i,
  & \Pi_i^\varnothing&=-\Pr[\text{needed}]\,\kappa_i,\label{eq:PiCnull}
\end{align}
with $\Vslow_i=\Vfast_i-\lambda_i\ell\pi/(1-\gamma e^{-\delta_i})$. Item
$i\to$ NVM iff $\Pi_i^N\ge\max(\Pi_i^R,\Pi_i^C,\Pi_i^\varnothing)$. The
NVM-vs-cloud break-even:
\begin{equation}\label{eq:benc}
  \boxed{\;\underbrace{\tfrac{\lambda_i\ell\pi}{1-\gamma e^{-\delta_i}}
    +(p_C-p_N)s_i}_{B_i:=\text{value of locality}}
  \;\ge\;(\cwear+\eta)\,w_i\;}\tag{BE-NC}
\end{equation}
and the NVM-vs-RAM break-even (the fallback v1 dropped):
\begin{equation}\label{eq:benr}
  \boxed{\;(\mu_R+p_R)s_i-p_N s_i\;\ge\;(\cwear+\eta)\,w_i\;}\tag{BE-NR}
\end{equation}
i.e.\ NVM beats RAM iff its wear cost is below RAM's capacity-rent premium. The
endurance price $\eta$ is fixed by the budget binding,
$\sum_t\sum_{i\in N_t(\eta)}w_i=\Eend$, a one-dimensional monotone root-find
($N_t(\eta)$ shrinks as $\eta\uparrow$).

Each tier return $\Pi_i^{x}$ reads as a per-item income statement: revenue is the
discounted value of recalls ($\Vfast_i,\Vslow_i$), against RAM rent, NVM occupancy
$p_N s_i$, cash-wear depreciation $\cwear w_i$, the cloud's storage-plus-latency
charge, or the write-off cost of forgetting. Placement is then capital budgeting,
with $\eta$ the hurdle rate the scarce flash imposes. Only the NVM line carries
depreciation \emph{plus} scarcity rent $\eta\,w_i$---that single extra term is what
makes the optimum non-monotone (\cref{sec:nonmono}).

\subsection{Headline: clean non-monotonicity}\label{sec:nonmono}
Parameterize a value stratum by $v$, holding $\lambda,\delta,s$ at conditional
means; within it $w$ has conditional mean $\wbar(v)$ and locality $B(v)$. Two
structural facts drive everything: (1) \emph{locality is $O(v^0)$}: with
$\lambda,\delta,s$ frozen at their conditional means, $B(v)$ is flat in $v$ to
first order ($B'(v)\approx 0$). Freezing $\lambda$ here is a modeling choice, and
a load-bearing one: if retrieval rate co-varies \emph{positively} with value, as
the very recurrence mechanism that yields $\chi>0$ would suggest (valuable scenes
are re-observed more often), then $B(v)$ rises in $v$ and the down-crossing
weakens. We therefore carry $|B'(v)|<(\cwear+\eta)|\chi|$ as an explicit clause of
the existence condition \eqref{eq:iff} rather than assume it away, and flag
$d\lambda/dv$ as an unmeasured primitive that bounds the result. (2) \emph{wear
cost is strictly increasing in $v$ iff $\chi>0$}: the \eqref{eq:benc} RHS has
$v$-derivative $(\cwear+\eta)\chi$.

\begin{definition}[endurance threshold]\label{def:etabar}
Let $v_{\max}$ be the largest $v$ with $\wbar(v)>0$ and set
\begin{equation}\label{eq:etabar}
  \boxed{\;\bar\eta:=\frac{B(v_{\max})}{\wbar(v_{\max})}-\cwear\;}
\end{equation}
the smallest erase-cycle price at which the wear term overtakes locality at the
top of the value support (assumed $>0$; else the headline is vacuous and the
Phase-0 gate kills it).
\end{definition}

\begin{proposition}[non-monotone-in-value optimum: clean conditions]\label{prop:nonmono}
Assume \cref{ass:value,ass:size,ass:tiers,ass:endurance,ass:recompute}, the
RAM-slack bound of \cref{prop:ramslack}, and the empirically-verified antecedent
\cref{ass:assoc} with $\chi>0$. Then for every endurance price $\eta>\bar\eta$
(equivalently, $\Eend$ below the level that induces $\bar\eta$), the optimal
local-persistence probability $\Pr[x=N\mid v]$---a probability rather than a hard
indicator, because within each value stratum the remaining primitives
$(\lambda,\delta,s)$ vary and smooth the per-item threshold into a curve---is
\emph{strictly non-monotone} in $v$: it rises on a low-value interval and strictly
falls on a high-value interval. The interior down-crossing $\vdc$ solves
$B(\vdc)=(\cwear+\eta)\wbar(\vdc)$ and satisfies
$\partial\vdc/\partial\eta<0$ and $\partial\vdc/\partial\chi<0$.
\end{proposition}
\noindent The down-crossing exists \textbf{iff}
\begin{equation}\label{eq:iff}
  \boxed{\;(\text{i})\ \eta>\bar\eta\ \text{AND}\
  (\text{ii})\ \chi=\tfrac{d}{dv}\E[w\mid v]>0,\
  \text{with}\ |B'(v)|<(\cwear+\eta)|\chi|\ \text{on the high-}v\ \text{interval}.\;}
\end{equation}
Neither condition alone suffices: (i) without (ii) is the monotone index
(\cref{prop:monotone}); (ii) without (i) gives $\eta=0$, wear bounded by
$\cwear\wbar(v)$, so for small $\cwear$ no down-crossing: the falsifiable
knife-edge the Phase-0 gate decides. \emph{(Proof in \cref{app:proofs}. Tested:
value-stratified persist regression, high-$v$ negative coefficient at $\eta>0$,
vanishing at $\Eend\to\infty$.)}

\subsection{The RAM-capacity multiplier survival bound}\label{sec:ramslack}
\begin{proposition}[down-crossing exists $\forall\mu_R$; location shifts left via endogenous $\eta$]\label{prop:ramslack}
The item ejected from NVM at $\vdc$ lands in RAM iff its wear cost exceeds
RAM's capacity-rent premium,
$(\cwear+\eta)\wbar(\vdc)>(\mu_R+p_R-p_N)s$; if $\mu_R$ is high enough that
the reverse holds for all $v$, the ejected item routes to cloud $\Pi^C$ or
recompute $\Pi^\varnothing$ (cheap by \cref{ass:recompute}). The down-crossing
\emph{exists for every} $\mu_R\ge 0$; $\mu_R$ controls only the \emph{destination}
of the spared endurance, with RAM-slack bound
$\mu_R<\mu_R^{\max}:=(\cwear+\eta)\wbar(\vdc)/s-(p_R-p_N)$. Because $\eta$ is
endogenous, raising $\mu_R$ pushes RAM-bound items into NVM (BE-NR), raising
erase demand and hence $\eta$, and since $\partial\vdc/\partial\eta<0$ the
crossing $\vdc(\mu_R)$ moves \emph{left} as $\mu_R\uparrow$.
\end{proposition}
\noindent\emph{(Proof in \cref{app:proofs}. Tested: RAM-pressure sweep;
down-crossing persists, $\vdc$ shifts left, destination RAM$\to$cloud as
$C_R\downarrow$.)}

\Cref{fig:phase} is the visual statement of
Props.~\ref{prop:monotone}--\ref{prop:ramslack}: the rise-then-fall persist
curve with the down-crossing $\vdc$ marked, the faint $\eta=0$ monotone reference, and the
placement-region backdrop.

\begin{figure}[t]
  \centering
  \includegraphics[width=0.86\linewidth]{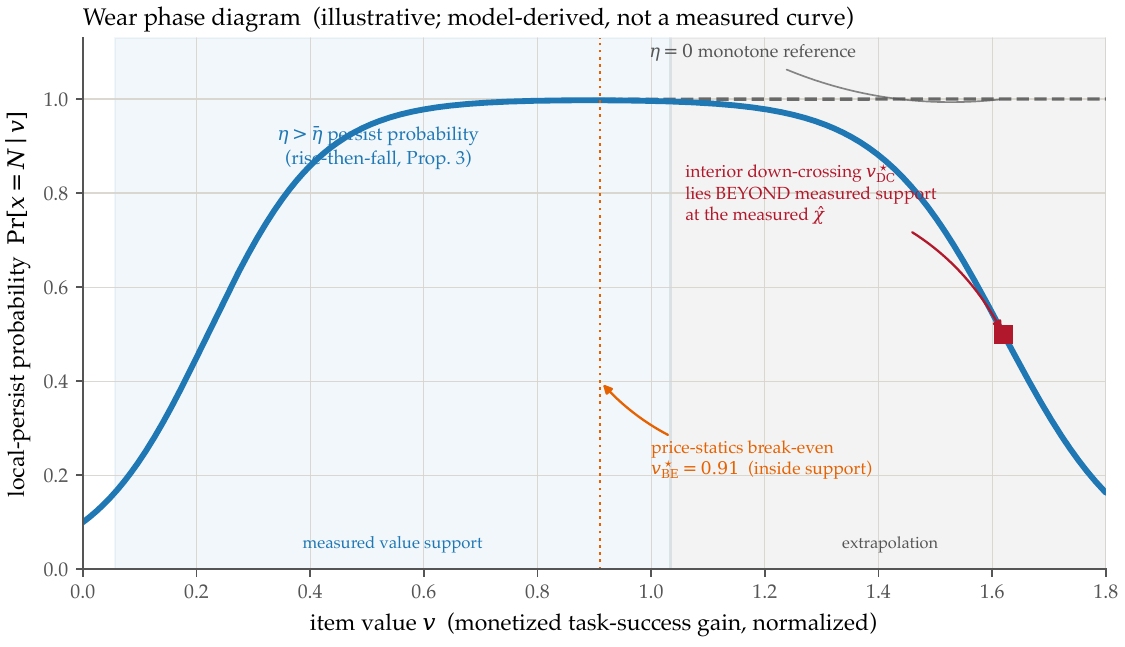}
  \caption{\textbf{Wear phase diagram (illustrative; model-derived).}
  Headline subject: the $\eta>\bar\eta$ persist-probability curve (rise-then-fall
  in value $v$), whose theory interior down-crossing $\vdc$ lies
  \emph{beyond} the normalized value support at the measured point estimate
  (\cref{sec:calibration}); the faint
  $\eta=0$ monotone step (\cref{prop:monotone}) is the contrast reference. The
  annotated marker is the price-statics \emph{break-even} $\vbe=0.91$ (the
  persist/evict crossover \emph{inside} support $[0.056,1.034]$); it and the
  ``do-NOT-persist-locally'' band boundary are
  read from the real price-statics overlay; the
  rise-then-fall persist-probability \emph{curve} is the closed-form shape of
  \cref{prop:nonmono} and is \emph{model-derived, not a measured empirical curve}. A measured
  value-stratified persist regression is not available: the offline H1 test (the
  non-monotone-persistence hypothesis) was vacuous at the non-binding S0 regime
  (today's base prices), and the binding-regime closed-loop eval
  (\cref{sec:results-binding}) returned a flat high-value persistence slope
  ($H1\ p=1.0$ in the primary cell), so it does not support the curve. We
  therefore present the curve as an illustration of the theory, not an empirical
  claim.}
  \label{fig:phase}
\end{figure}

\subsection{Hotelling: a bounded caveat}\label{sec:hotelling}
The endurance constraint is formally an exhaustible-stock
problem~\citep{hotelling1931economics}, but that scaffold justifies one fact only:
an erase cycle spent today is unavailable tomorrow, so its user cost carries a
present-value scarcity rent $\eta$---and this rent, not the cash wear $\cwear$,
makes placement multi-period. We do \emph{not} claim the path
$\eta_t=\eta_0(1+r)^t$ (with a positive cash term only the rent grows at $r$;
endurance is partially recoverable via wear-leveling; non-stationary retrieval
makes $\eta$ non-closed-form ex ante). The operative object is the dual of a
finite-horizon constrained (PO)MDP, estimated as $\widehat\eta$. As noted in
\cref{sec:rw-econ}, pricing memory as depreciating capital is not new in the token
setting; what is new is that $\eta$ here is the dual of a \emph{physical}
non-renewable P/E stock (\cref{ass:endurance}) under multi-tier energy-budgeted
placement.

\subsection{cs.CY corollary: endurance rent \texorpdfstring{$\to$}{->} device lifetime \texorpdfstring{$\to$}{->} fleet e-waste}\label{sec:cscy}
The rent $\eta$ is also a \emph{device-lifetime price}. Spending an erase cycle
consumes a fixed fraction $1/\Eend$ of NAND life; at the budget binding the
policy's cumulative erases map one-to-one to calendar lifetime.

\begin{corollary}[bounded lifetime lever]\label{cor:lifetime}
Lowering fleet-wide NVM-erase demand by a fraction $q$ (the controller's saving
vs.\ an endurance-blind baseline) extends flash-limited device life by
$\approx q$ to first order, deferring replacement embodied carbon at the
$\approx 22$~kg~CO$_2$e/TB NAND anchor~\citep{weppe2025nandcarbon}. This is a
\emph{directional, not magnitude-certified} lever: replacement is multi-causal;
SSDs are the dominant and growing component of device
carbon~\citep{gupta2022act,schneider2025lifecycle}, end-of-use $\ne$
end-of-life, and lifetime-extension can overstate real
savings~\citep{bashir2023embodiedpitfalls,switzer2022junkyard}.
\end{corollary}
\noindent So $\eta$ wires the core model to a falsifiable cs.CY claim,
\emph{cost-optimal forgetting is also lifetime-extending}, without a precise
carbon number (reported as a TCO/carbon implication, not an experimental
target).

\subsection{Price comparative statics}\label{sec:pricestatics}
Placement is a function of the cloud-tier price vector
$\mathbf p=(p_{\mathrm{HBM}},p_{\mathrm{DDR}},p_{\mathrm{NAND}},
p_{\mathrm{egress}},p_{\mathrm{energy}})$: RAM rent $p_R$ tracks DDR/LPDDR
\$/GB; $\cwear,p_N$ track enterprise NAND \$/GB and \$/P-E-cycle; $p_C$ and
$\ell\pi+e_C$ track egress \$/GB and energy \$/kWh; $\eta$ is endogenous. We sign
the statics over a three-point band (\cref{tab:priceband}), never a point
estimate, since memory prices are oligopoly/contract-set
\citep{trendforce2025asp,counterpoint2025serverdram}. The same Low/Base/High
band is used everywhere a price enters the paper.

\begin{table}[t]
\centering
\caption{Unified price scenario band, used identically in the model, the
calibration, and the pre-specified experiment plan. Anchors:
\citet{counterpoint2025serverdram,trendforce2025asp,elinfor2026nand}.}
\label{tab:priceband}
\small
\setlength{\tabcolsep}{4pt}
\begin{tabular}{lcccl}
\toprule
Scenario & DDR5 \$/GB & NAND TLC \$/GB & egress \$/GB & basis\\
\midrule
S0 $\equiv$ Base (today)    & $\sim 9$  & $\sim 0.13$ & $\sim 0.09$ & current spot \citep{counterpoint2025serverdram}\\
S1 $\equiv$ High (DRAM cycle)   & $\sim 16$ & $\sim 0.18$ & $\sim 0.09$ & DRAM-led tightening \citep{trendforce2025asp}\\
S2 $\equiv$ High (NAND shock)   & $\sim 12$ & $\sim 0.22$ & $\sim 0.09$ & NAND-led tightening \citep{elinfor2026nand}\\
Low (floor)            & $3$       & $0.06$      & $0.09$      & low-price boundary\\
\bottomrule
\end{tabular}
\end{table}

\begin{proposition}[signed statics P5a--d]\label{prop:statics}
Writing $\sigma_N$ for the population NVM share and $\Theta$ for the NVM-persist
value-threshold, each sign follows from \eqref{eq:benc}/\eqref{eq:benr} plus the
budget identity fixing $\eta$:
\begin{description}[leftmargin=2.2em,itemsep=2pt]
\item[\textnormal{P5a}] $p_{\mathrm{NAND}}\uparrow$: persist region shrinks,
$\partial\sigma_N/\partial p_{\mathrm{NAND}}<0$ and
$\partial\eta/\partial p_{\mathrm{NAND}}<0$ (NVM$\downarrow$, $\eta\downarrow$; clean).
\item[\textnormal{P5b}] $p_{\mathrm{DDR}}\uparrow$: \emph{conditional on an
interior RAM allocation}, (BE-NR) tilts toward NVM and
$\partial\eta/\partial p_{\mathrm{DDR}}\ge 0$. \textbf{Under binding RAM the
static is (near-)zero}: the RAM shadow price $\mu_R$ absorbs the entire
$p_{\mathrm{DDR}}$ shock one-for-one, so it does not propagate to the endurance
margin ($d\eta\approx 0$). Empirically this is the operative regime
(\cref{sec:results}), so P5b returns an \emph{inconclusive null}: a clean
identification of the boundary condition under which the proven sign is
observable, not a contradiction.
\item[\textnormal{P5c}] $p_{\mathrm{egress}},p_{\mathrm{energy}}\uparrow$: cloud
dearer, $B(v)\uparrow$, more onboard, $\eta\uparrow$ ($\sigma_N\uparrow$ first
order; threshold move directional, not magnitude-signed).
\item[\textnormal{P5d}] $\Eend\uparrow$: $\partial\eta/\partial\Eend<0$; as
$\eta\to 0$ we recover \cref{prop:monotone} and non-monotonicity vanishes
(clean boundary; a small-endurance phenomenon).
\end{description}
\end{proposition}
\noindent\emph{(Proof in \cref{app:proofs}. Tested: price-band sweep,
re-solve $\eta(\mathbf p),\sigma_N(\mathbf p)$ at S0/S1/S2.)}

\paragraph{A falsifiable conjecture (cross-partial).}
Bandwidth $b$ enters only $B(v)$, so extra bandwidth diverts writes to cloud,
relaxes endurance, lowers $\eta$, and lifts the value of remaining onboard items.
This predicts
\begin{equation}\label{eq:xpartial}
  \frac{\partial^2(\text{fleet task value})}{\partial b\,\partial(1/\Eend)}\;>\;0,
\end{equation}
i.e.\ the marginal task-value of bandwidth rises with endurance tightness---for
memory-bound fleets on cheap NAND, buying radio is partly buying flash lifetime. We
state this as a model-predicted conjecture (H6, \cref{tab:hypotheses}), not a
theorem: it is not derived in \cref{prop:statics}, and its mechanism has two
opposing channels. The empirical sweep (\cref{sec:results-price}) returns $+0.50$
with a bootstrap CI that straddles zero ($[-0.34,+1.25]$), so we report it as
directional, not confirmed.

\subsection{Calibration: measured primitives in the model}\label{sec:calibration}
We plug the \emph{measured} primitives ($\hat\delta,\etamark,\hat\chi$) and the
datasheet cost constants into \crefrange{eq:Vlocal}{eq:xpartial}. \textbf{Scope (binds
every number in this subsection):} the canonical $\hat\chi=+1.016\times10^{-3}$, $95\%$
physical-scene-clustered CI $[+3.81\times10^{-4},+1.65\times10^{-3}]$ ($n{=}3{,}032$,
$379$ clusters), is the \textbf{LIBERO-LONG, SmolVLA-0.5B} headline; on a
non-recurrent teleoperation distribution \textbf{DROID measures a significantly
\emph{negative}} slope at high power ($\hat\chi=-8.95\times10^{-3}$, CI
$[-1.61\times10^{-2},-4.09\times10^{-3}]$), a measured opposite-sign regime, so
every calibrated \emph{association} magnitude is recurrent-regime-scoped.
$\hat\delta=0.032$/step (half-life $21.7$ steps, CI $[16.2,29.6]$, $R^2{=}0.997$) is the
DROID recurrence-kernel anchor, used \emph{only as a value-decay timescale}. \textbf{We
mix a DROID-anchored $\hat\delta$ with the LIBERO-scoped $\hat\chi$ and flag the mix
as such}: we do \emph{not} claim $\delta$ is regime-robust. The LIBERO state-kernel
$\delta\approx0.0004$ measures state \emph{persistence}, not value depreciation, so it
cannot serve; the LIBERO \emph{value-decay} $\delta\approx0.054$ (sparse, $R^2{=}0.15$)
is the same-regime quantity but is $\sim$$69\%$ larger than the DROID $0.032$ and, used
instead, would give $M\approx16\times$ (half-life $\approx13$ steps) rather than
$24.3\times$. We headline the DROID-anchored value $M=24.3\times$ for its far tighter
fit ($R^2{=}0.997$ vs.\ $0.15$) but flag that this is a fit-quality choice, not evidence
of cross-regime stability of the value-decay timescale; the LIBERO value-decay
$M\approx16\times$ is the same-regime robustness alternative.
$\etamark=0.0503$ is the normalized endurance scarcity markup solved on the
pooled Phase-0 logs under a capped budget (binds at $\le 0.75\times$ measured
write demand, flat across binding budgets); at uncapped datasheet conditions the
budget does not bind and the markup is zero. \textbf{Two duals, two unit systems:} $\etamark$ is the Phase-0 normalized
scarcity-markup ratio ($+5.0\%$ on the cash wear of one erase; Jorgenson user cost
$=\cwear(1+\etamark)$, a unit-free indicator that endurance binds), while
$\etasim\approx2.4\times10^{-4}$ in the price-statics results
(\cref{sec:results-price}) is a separately-solved equilibrium dual in the
simulator's objective units. They are the same \emph{kind} of object---the dual on
$\sum\sum w\le\Eend$---in two different objectives, so we do not compare them
numerically; the bare symbol $\eta$ is reserved for the model-theoretic shadow
price.

\paragraph{Depreciation and the locality multiplier.} With $\gamma=0.99$,
$D=1-\gamma e^{-\hat\delta}=0.0412$ (CI $[0.0329,0.0516]$), so a persisted item is worth
$M=1/D=\mathbf{24.3\times}$ its per-step value (CI $[19.4,30.4]$).

\paragraph{Wear cost and the endurance user cost.} The datasheet anchors (TLC
$3{,}000$~P/E, WAF~$3$, $128$~GB module) give TBW${=}128$~TB and a base cash wear of
$\mathbf{\$0.13}$~per~TB-written, i.e.\ $5.55\times10^{-6}$~\$/(block P/E). The Jorgenson
user cost of one erase, $\cwear(1+\etamark)$, is $\mathbf{5.83\times10^{-6}}$
\$/(block P/E) at base NAND (band $[2.69,9.86]\times10^{-6}$ over the Low/Base/High
NAND price band, \cref{tab:priceband}); the rent, not the cash term, makes
placement multi-period.

\paragraph{Calibrated break-even (BE-NR).} \eqref{eq:benr} with slack RAM
($\mu_R{=}0$) and the base DDR--NAND gap ($\$8.87$/GB capacity-rent premium) yields:
\emph{$\approx \$4.8\times10^{-5}$ of task value per GB-day justifies NVM persistence at
base NAND} ($\$0.13$/TB cash wear $+\,5\%$ rent). The RAM premium exceeds wear cost by
$\sim 5$ orders of magnitude: at measured intensities capacity rent, not wear, evicts
from RAM; wear bites only once endurance binds ($\eta>0$), the regime we isolate.

\paragraph{Calibrated down-crossing $\vdc$: outside support.} Our data do not pin
the locality level $B$ independently, so we do not headline a number for $\vdc$ (the
point estimate $\vdc=\chi_{\mathrm{hi}}/\hat\chi\approx1.62$ is an anchoring
identity, not a $B$-grounded measurement). The anchoring-invariant conclusion is
that \textbf{at the measured $\hat\chi$ the interior down-crossing lies beyond the
measured value support}, reaching the support edge only at the upper end of the
$\chi$-CI. The non-monotone optimum is therefore sign-correct but
\textbf{quantitatively dormant at measured intensities}, becoming empirically live
only under tighter endurance. It is a distinct object from the price-statics
break-even $\vbe{=}0.91$ (\cref{sec:results-price}). This is the headline
calibration caveat.

\paragraph{Calibrated cs.CY lever (bounded).} The measured write-intensity spread
($\approx 9.5\%$ of mean) bounds the divertible erases: device-life extension under the
index policy vs.\ LRU (\cref{cor:lifetime}) ranges from $\approx 0\%$ at the measured
$\hat\chi$ (since $\vdc$ is beyond support) up to a $\le 9.5\%$ ceiling when endurance
is tight enough to pull $\vdc$ interior, deferring $\le 2.82$~kg~CO$_2$e per $128$~GB
module life at the Weppe anchor~\citep{weppe2025nandcarbon}: directional, not
magnitude-certified.

\paragraph{Calibrated cross-partial (H6).} The conjecture of \eqref{eq:xpartial}
calibrates to a directional, CI-inconclusive estimate: flipping one item from NVM to
cloud spares $\bar w\approx 1.05$ P/E-cycles ($\approx 2.9\times10^{-7}$\$\ of
endurance relief at base NAND), and the swept price-grid estimate is $+0.50$ with CI
$[-0.34,+1.25]$ (\cref{sec:results-price}). We report it as directional---for
memory-bound fleets on cheap NAND, buying radio is partly buying flash
lifetime---not confirmed; full CI propagation ships in the artifact
(\cref{app:repro}).

\begin{figure}[t]
  \centering
  \includegraphics[width=0.92\linewidth]{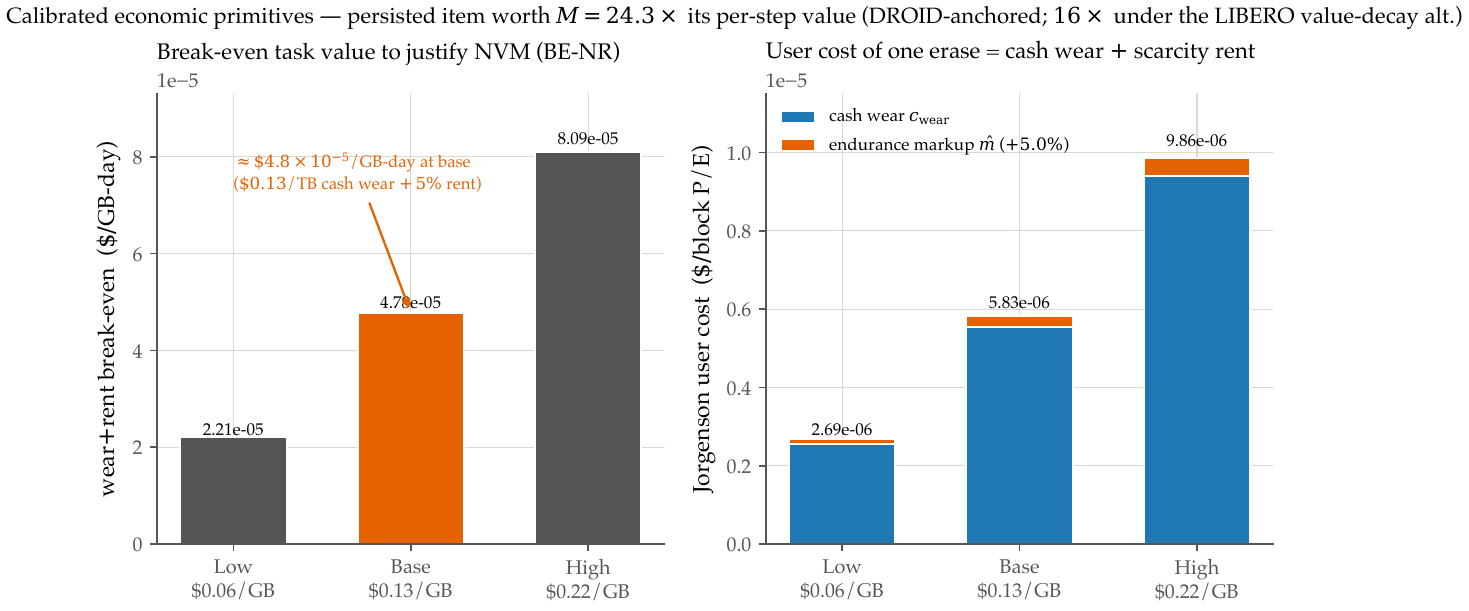}
  \caption{\textbf{Calibrated economic primitives.} Left: the wear$+$rent
  break-even ($\$/$GB-day) at Low/Base/High NAND ($\approx\$4.8\times10^{-5}$ at
  base). Right: the Jorgenson user cost of one erase, decomposed into datasheet
  cash wear plus the $+5.0\%$ endurance markup ($\etamark$). The persisted
  item is worth $M=24.3\times$ its per-step value (half-life $21.7$ steps).
}
  \label{fig:calib}
\end{figure}

\section{Experimental Design}\label{sec:design}
The design is pre-specified as a five-phase pipeline with published kill
criteria (\cref{fig:pipeline}); the frozen plan ships in the released artifact
(\cref{app:repro}). One scope note on what ``pre-specified'' means here: the
plan is a frozen, version-controlled document committed before the corresponding
runs, with per-phase gates and kill criteria fixed in advance, but it is
\emph{internally} timestamped rather than lodged with an external registry
(e.g.\ OSF), a weaker guarantee that we state plainly. Headline
backbone: SmolVLA-0.5B~\citep{shukor2025smolvla}; OpenVLA-7B
\citep{kim2024openvla} is a single \emph{scale-stress} arm (the pre-specified
cross-backbone confirmation requires Spearman $\ge 0.6$, which it does not meet;
see \cref{sec:results}). Datasets: LIBERO-LONG~\citep{liu2023libero}, DROID,
a $100$-ep Phase-1 sample enlarged for a high-power cross-dataset re-test,
analyzed on the pre-specified $1{,}200$-new-only subset ($n=9{,}598$ frames,
$1{,}200$ physical-scene clusters; \texttt{lerobot/droid\_1.0.1})~\citep{khazatsky2024droid}. Primary metric:
\emph{task-success-per-joule-per-erase}.

\begin{figure}[t]
\centering
\resizebox{\linewidth}{!}{%
\begin{tikzpicture}[font=\scriptsize, >={Stealth[length=2mm]},
  ph/.style={draw, rounded corners=1.5pt, align=center, inner sep=4pt,
             line width=0.7pt, text width=2.3cm, minimum height=1.6cm},
  kg/.style={align=center, text=rentred, font=\tiny, text width=2.5cm,
             inner sep=1pt},
  flow/.style={->, line width=0.8pt, draw=inkgray}]
\node[ph, fill=ramblue!8, draw=ramblue!75!black] (p0) at (0,0)
  {\textbf{Phase 0}\\value--write gate\\$\hat\chi$, $\rho_s$ on real robot logs ($\le\$25$)};
\node[ph, fill=cloudteal!8, draw=cloudteal!75!black] (p1) at (3.05,0)
  {\textbf{Phase 1}\\value labeling\\SmolVLA-0.5B counterfactuals $+$ value model};
\node[ph, fill=nvmorange!10, draw=nvmorange!80!black] (p2) at (6.10,0)
  {\textbf{Phase 2}\\placement controller\\set transformer, $3.15$M\\BC warm-start $+$ PPO};
\node[ph, fill=ramblue!8, draw=ramblue!75!black] (p3) at (9.15,0)
  {\textbf{Phase 3}\\offline eval battery\\cost-matched ladder\\McNemar $+$ Holm};
\node[ph, fill=rentred!7, draw=rentred!70!black] (w3) at (12.20,0)
  {\textbf{Closed loop /}\\\textbf{VLA-in-the-loop}\\binding-regime\\rollouts $+$ causal gate};
\draw[flow] (p0) -- (p1);\draw[flow] (p1) -- (p2);
\draw[flow] (p2) -- (p3);\draw[flow] (p3) -- (w3);
\node[kg, below=3pt of p0] {kill: $\rho_s\in[-0.1,+0.1]$ $\wedge$ recompute $<10\%$ (passed)};
\node[kg, below=3pt of p1] {floor: cross-backbone $\rho_s\ge0.6$ (not met: 7B uninterpretable)};
\node[kg, below=3pt of p2] {cap: $\le 50$M params\\$\ge 3$ seeds (met)};
\node[kg, below=3pt of p3] {H1--H6 kill criteria;\\H2 negative};
\node[kg, below=3pt of w3] {causal gate $\ge+8$\,pp, measured $+0.0$\,pp (abort)};
\end{tikzpicture}}
\caption{\textbf{Experiment-architecture pipeline.} Five pre-specified stages,
each with a published kill criterion (red): the \$25 Phase-0 gate decides whether
the non-monotone branch is admissible before any training spend; Phase 1 labels
item value via action counterfactuals; Phase 2 trains the $3.15$M-parameter
placement controller; Phase 3 runs the cost-matched McNemar/Holm battery; the
final stage closes the loop in the binding endurance regime, and the
VLA-in-the-loop causal gate aborted at $+0.0$\,pp.}
\label{fig:pipeline}
\end{figure}
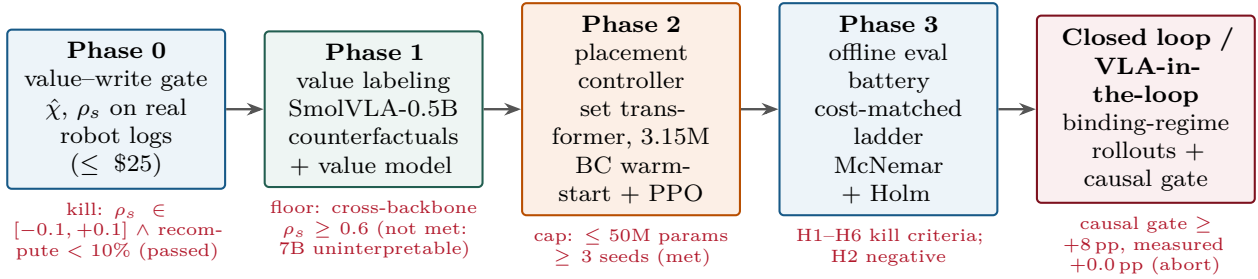

\subsection{Phase 0: value--write go/no-go gate ($\le\$25$, pre-specified)}
Measures whether write-intensity $w_i$ is associated with item value $v_i$
through a channel that is \emph{not} the shared retrieval-frequency process. The
value proxy $\hat v_i$ is counterfactual task-outcome attribution on an
\emph{independent} sample via batched SmolVLA masked-vs-unmasked passes, with a
forward-value regression confirmation; $w_i$ is measured on the complement
sample. We pre-specify \emph{both} estimators: Spearman $\rho_s(w,\hat v)$
(rank screen) and the local-slope $\hat\chi=\tfrac{d}{dv}\E[w\mid v]$
(kernel/local-linear, percentile-bootstrap 95\% CI); $\hat\chi$ is
decision-relevant. \textbf{Kill criterion:} pivot to the monotone-index design
iff (proxy spot-check passed) AND $\hat\chi\le 0$ with CI excluding $\chi>0$ AND
$\rho_s\in[-0.1,+0.1]$ AND cheap-recompute fraction $<10\%$, on \emph{both}
LIBERO and DROID. The gate is deliberately \emph{conjunctive across datasets};
\cref{sec:results} reports that DROID alone later met these thresholds while
LIBERO did not, so the AND-gate correctly did not fire and the non-monotone
branch is retained but \emph{scoped} to the recurrent regime.

\subsection{Phase 1: value labeling}
\begin{sloppypar}
P1a: batched counterfactual labeling of a stratified item subset with frozen
SmolVLA-0.5B (importance-weighted to the population; exact \#forward-passes
pre-specified). P1b: a $\le 50$M supervised value model
$(\hat v_i,\hat\lambda_i,\hat\delta_i)$ generalizing sparse labels. P1c:
OpenVLA-7B single scale-stress arm (pre-specified cross-backbone Spearman
$\ge 0.6$ target), \emph{not} a per-step counterfactual sweep.
\end{sloppypar}

\subsection{Phase 2: placement controller ($\le 50$M; BC warm-start + PPO; $\ge 3$ seeds)}
A small transformer over the item set $+$ a system-state token. Reward
$=$ success-proxy $-\lambda\cdot$energy $-\nu\cdot$erase $-\zeta_{\mathrm{cloud}}
\cdot$cloud-\$ (the cloud weight is $\zeta_{\mathrm{cloud}}$; $\rho$ is
reserved for the Phase-0 Spearman, $\chi$ for the value--write slope; no symbol
does double duty: in particular the two value crossings carry distinct symbols
throughout: $\vbe$ denotes the price-statics \emph{break-even} on support
($=0.91$, \cref{sec:results-price}) and $\vdc$ the theory interior
\emph{down-crossing} (beyond support at the measured $\hat\chi$, \cref{prop:nonmono}); the
bare word ``down-crossing'' always refers to $\vdc$). Hindsight-oracle warm start, then PPO over stochastic
network/energy/endurance regimes; split by scene (70/15/15).

\subsection{Phase 3: baseline ladder, $\delta$ fitting, sensitivity}
\Cref{tab:baselines} is the cost-matched baseline ladder. $\delta$ is fitted
from retrieval-recurrence decay (reported fitted-vs-assumed). A 4-axis
stitched-episode sensitivity battery (stitch-boundary incl.\ adversarial,
depreciation, stream-length, scope-sensitivity) is the headline robustness claim.

\begin{table}[t]
\centering
\caption{Cost-matched baseline ladder. Each baseline is constrained to the
identical energy $+$ erase $+$ \$ budget as ours.}
\label{tab:baselines}
\small
\setlength{\tabcolsep}{4pt}
\begin{tabular}{ll p{0.32\linewidth}}
\toprule
Tag & Baseline & Role\\
\midrule
all-cloud / all-RAM / all-NVM & extremes & bound frontier; all-NVM hits the cliff\\
LRU / LFU / size-based & classic caching & monotone-in-score floor\\
Flashield-style~\citep{flashield2019nsdi} & ML admission, cost-matched & pre-empt rebuttal\\
CacheSack-style~\citep{cachesack2023tos} & knapsack admission, cost-matched & A3 rebuttal\\
Chinchali offload~\citep{chinchali2019offloading} & learned offload, cost-matched & ``just offloading'' rebuttal\\
surprise-gated~\citep{gorlo2026surprise} & embodied-memory heuristic & SOTA heuristic arm\\
AURA single-tier~\citep{chen2026aura} & write-if-gated, no tier choice & \textbf{must strictly dominate}\\
random / hindsight oracle & floor / ceiling & frontier bounds\\
\bottomrule
\end{tabular}
\end{table}

\subsection{Pre-specified hypotheses}\label{sec:hypotheses}
Eval $N\ge 200$ held-out episodes/seed ($\ge 600$ paired at the 3-seed floor;
5 seeds for headline H1/H3/H4). Paired McNemar vs.\ each cost-matched baseline,
Holm--Bonferroni over the family \{H1, H1b, H2, H3, H4, H4b, H5, H6\}. Three of
these (H4, H4b, H5) require a swept-box / $C_R$-sweep run that did not land, so the
\emph{realized} corrected family is \{H1, H1b, H2, H3, H6\}; we flag the unrun
members rather than quietly drop them. Every value--write test pre-specifies both
$\rho_s$ (screen) and $\hat\chi$ (decision-relevant). Statistics methodology
follows~\citet{chen2026aegis}.
\Cref{tab:hypotheses} lists each falsifiable statement and its kill criterion.

\begin{table}[t]
\centering
\caption{Pre-specified hypotheses, tests, and kill criteria.}
\label{tab:hypotheses}
\footnotesize
\begin{tabular}{>{\RaggedRight\arraybackslash}p{0.05\linewidth}>{\RaggedRight\arraybackslash}p{0.40\linewidth}>{\RaggedRight\arraybackslash}p{0.22\linewidth}>{\RaggedRight\arraybackslash}p{0.23\linewidth}}
\toprule
ID & Falsifiable statement & Test / metric & Kill criterion\\
\midrule
H1 & Optimal local-persist is non-monotone in $v$ when $\eta>0$ (negative high-$v$ coeff). & persist regression; sign $+$ $p$ (Holm). & no negative high-$v$ coeff at $p<0.05$ in any $\eta>0$ regime.\\
H1b & Persist set monotone-shrinking as endurance tightens (Prop.~\ref{prop:monotone}). & $\partial(\text{persist frac})/\partial(1/\Eend)$; bootstrap CI. & not decreasing in $1/\Eend$ at $p<0.05$.\\
H2 & Controller beats every cost-matched baseline on success-per-joule-per-erase. & McNemar, Holm. & fails strongest cost-matched baseline at $p<0.05$.\\
H3 & Controller strictly dominates the AURA single-tier arm. & McNemar per-joule-per-erase. & no strict dominance at $p<0.05$.\\
H4 & Non-monotonicity occupies $\ge 10\%$ of the datasheet-plausible box. & fraction of swept box with down-crossing. & $<10\%$ of the box.\\
H4b & $\vdc$ exists $\forall\mu_R$; location shifts with $\eta$ (Prop.~\ref{prop:ramslack}). & down-crossing per $\mu_R$ cell; $\partial\vdc/\partial\eta<0$. & $\vdc$ absent in a cell OR slope not negative.\\
H5 & Boundary sensitive to $\Eend$ (vanishes as $\Eend\to\infty$). & $\partial\Theta/\partial\Eend<0$. & boundary insensitive to $\Eend$.\\
H6 & Marginal task-value of bandwidth increases in endurance tightness. & $\partial^2(\text{value})/\partial b\,\partial(1/\Eend)>0$. & flat in tightness (non-fatal).\\
\bottomrule
\end{tabular}
\end{table}

\section{Results}\label{sec:results}
All $\chi$ are computed with the verbatim Phase-0 local-linear-slope estimator
under \emph{physical-scene clustering} (resampling physical \texttt{scene\_id},
not seed-reshuffle pseudo-clusters) with $1{,}000$-resample $95\%$ scene-clustered
bootstrap CIs; the canonical $\chi$ table shipped in the artifact
(\texttt{chi\_canonical\_table.json}, \cref{app:repro}) is the paper's source of
truth. Controller and price-statics results are reported exactly as the data say
them; a negative or null result is reported as a finding, not a failure.

\subsection{Phase-0 gate and the \texorpdfstring{backbone$\times$regime $\chi$}{backbone-regime chi} matrix}\label{sec:results-chi}
\Cref{tab:chimatrix} is the canonical value$\to$write matrix. On LIBERO-Long with
the SmolVLA-0.5B backbone, $\hat\chi=+1.016\times10^{-3}$ (CI
$[+3.81\times10^{-4},+1.65\times10^{-3}]$, $\rho_s=+0.10$, cheap-recompute fraction
$<10\%$): it excludes zero positive and survives Holm correction across the
three-arm $\chi$ family (rank 2, $\alpha=0.025$, reject). The cheap forward-value
proxy tracks the full SmolVLA counterfactual at held-out $r_{\mathrm{cf}}=0.92$
(pre-specified floor $0.40$)---an \emph{internal} proxy-consistency check, not a
validation of $\hat v$ against realized task success, which the $+0.0$pp causal gate
(\cref{sec:honesty}) leaves open; every $\chi$ here is therefore a coupling to the
value \emph{proxy}. The
rejection does not depend on the family's composition---it holds as a single test,
in a two-arm pre-specified family, and at rank 2 of the three-arm family---so the
inclusion of the post-hoc DROID arm, whose unadjusted significance we read only as
exploratory, does not carry the headline.
The effect is \emph{real but small}: the $\chi$-implied $w$-swing across the
$5$th--$95$th value percentile is $4.92\times10^{-4}$, $\approx5.6\%$ of $w$'s
dynamic range, so it supports the \emph{conditional} branch of
\cref{prop:nonmono}, not a universal monotone law. The effect is \emph{LIBERO-Long-specific}: a second,
shorter-horizon LIBERO suite (goal/object/spatial) is null ($\hat\chi=-1.58\times10^{-3}$,
CI $[-3.84\times10^{-3},+6.5\times10^{-4}]$ straddles zero, $n{=}1{,}520$ items /
$190$ clusters, $\rho_s\approx0$; second-suite row of the canonical $\chi$
table), so the coupling tracks long-horizon
recurrence, not LIBERO-as-a-dataset. The enlarged DROID arm is significantly
negative ($\hat\chi=-8.95\times10^{-3}$, CI $[-1.61\times10^{-2},-4.09\times10^{-3}]$,
Holm rank 1, reject) but \emph{post-hoc / exploratory}: the enlargement was
launched after a negative, underpowered $100$-scene pilot (a forking path), so we
report it on the pre-specified $1{,}200$-new-only subset (excluding the motivating
episodes) as a suggestive regime-difference signal, not a confirmation.

\paragraph{The OpenVLA-7B arm is uninterpretable, not a disconfirmation.}
The pre-specified confirmation criterion is a cross-backbone Spearman
$\ge0.6$ between the two backbones' per-item value rankings. Measured on the
identical $3{,}032$ LIBERO frames it is only $\rho_s=0.05$ (CI $[0.016,0.086]$),
far below the floor: OpenVLA-7B and SmolVLA place items on \emph{near-orthogonal
value axes}, so a $\chi$ sign difference between them is two incomparable
measurements, not a contradiction of the headline mechanism. Independently, under
physical-scene clustering OpenVLA's own $\chi$ straddles zero ($-2.44\times10^{-4}$,
$p=0.18$). We therefore withdraw any cross-backbone ``sign reversal'' claim and
report the arm as null/uninterpretable (\cref{fig:chimatrix}).

\paragraph{Gate decision.} The pre-specified kill is conjunctive across datasets:
DROID independently met the kill thresholds ($\hat\chi\le0$ with CI excluding
$\chi>0$, $\rho_s=-0.077\in[-0.1,+0.1]$), but LIBERO-Long did not, so the AND-gate
did \emph{not} fire. We \textbf{proceed} on the non-monotone branch, scoped to the
recurrent regime per the measured boundary.

\begin{table}[t]
\centering
\caption{Canonical value$\to$write matrix (physical-scene clustering,
$1{,}000$-resample bootstrap, Holm across the three-arm $\chi$ family). Source:
\texttt{chi\_canonical\_table.json}.}
\label{tab:chimatrix}
\small
\setlength{\tabcolsep}{4pt}
\begin{tabular}{lllrrl}
\toprule
Backbone & Regime / suite & $\hat\chi$ (raw) & $\beta_{\mathrm{std}}$ & $\rho_s$ & Verdict\\
\midrule
SmolVLA-0.5B & LIBERO-Long (recurrent) & $+1.016\!\times\!10^{-3}$ & $+0.118$ & $+0.10$ & CI$>0$; Holm-reject\\
SmolVLA-0.5B & LIBERO goal/obj/spatial & $-1.58\!\times\!10^{-3}$ & $-0.064$ & $\approx0$ & straddles 0; null\\
SmolVLA-0.5B & DROID (enlarged, post-hoc) & $-8.95\!\times\!10^{-3}$ & $-0.084$ & $-0.077$ & CI$<0$; Holm-reject\\
OpenVLA-7B & LIBERO-Long (scale-stress) & $-2.44\!\times\!10^{-4}$ & $-0.094$ & $-0.006$ & straddles 0; uninterp.\\
\bottomrule
\end{tabular}
\end{table}

\begin{figure}[t]
  \centering
  \includegraphics[width=0.92\linewidth]{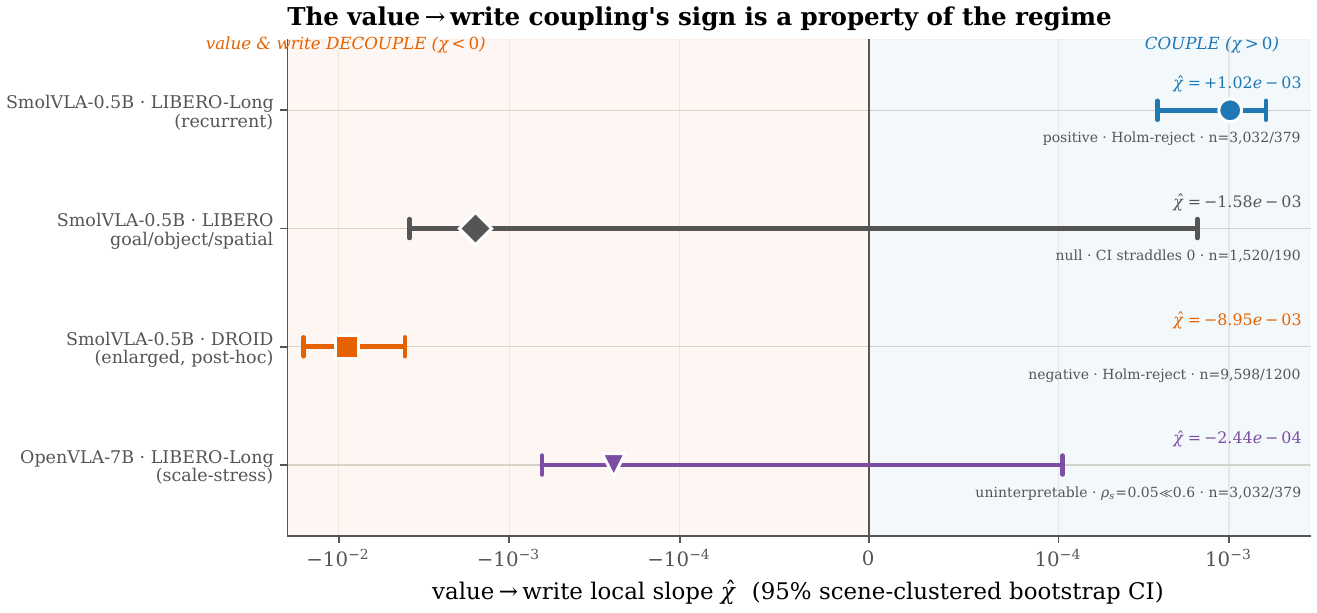}
  \caption{\textbf{The value$\to$write coupling's sign is regime-dependent.}
  Canonical $\hat\chi$ with $95\%$ scene-clustered CIs: positive and CI-excluding-zero
  on recurrent LIBERO-Long (SmolVLA-0.5B), significantly negative on non-recurrent
  DROID (post-hoc), and CI-straddling on the OpenVLA-7B scale-stress arm, whose
  cross-backbone agreement ($\rho_s=0.05\ll0.6$) makes its sign uninterpretable.
}
  \label{fig:chimatrix}
\end{figure}

\begin{figure}[t]
  \centering
  \includegraphics[width=0.92\linewidth]{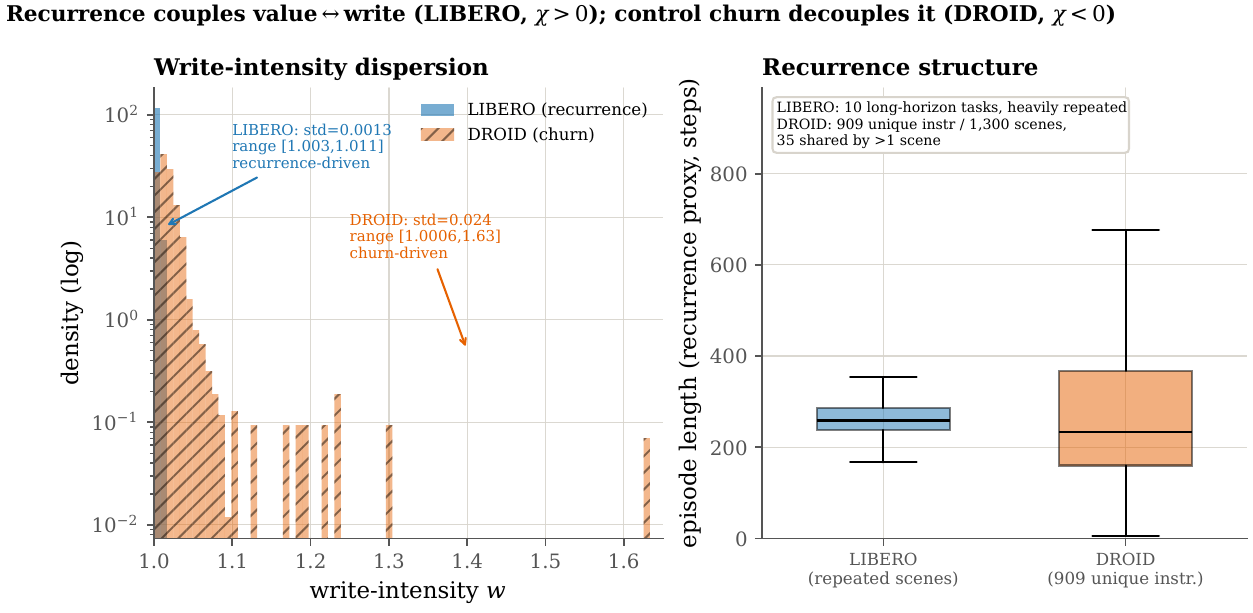}
  \caption{\textbf{Recurrence-driven vs.\ churn-driven write-intensity: the
  mechanism behind the sign.} Left: $w$ is tight on LIBERO-Long
  ($\mathrm{std}\,0.0013$, dominated by re-observation of recurrent scenes) and
  wide on DROID ($\mathrm{std}\,0.024$, range $1.0006$--$1.63$, driven by
  teleoperation churn). Right: episode-length dispersion, a recurrence proxy
  ($268\pm57$ steps LIBERO vs.\ $303\pm230$ DROID; $909$ unique instructions over
  the $1{,}300$-scene full DROID pool. The $\chi$ estimate is reported on the
  pre-specified $1{,}200$-new-only analyzed subset: $1{,}300$ is the full pool,
  $1{,}200$ the analyzed subset; the difference is the $100$-scene motivating pilot
  excluded to avoid the forking path).}
  \label{fig:regime}
\end{figure}

\subsection{Recurrence dose-response (mechanism test)}\label{sec:results-dose}
To test the recurrence mechanism directly, we interpolate episode mixes blending
the non-recurrent DROID regime ($\chi<0$) with the recurrent LIBERO-Long regime
($\chi>0$), with dose $=$ fraction of recurrent episodes (\cref{fig:dose}). The
pre-specified design is a twelve-level, three-seed sweep across the sign-flip band.
At its first run the DROID frame loader drew from a $100$-episode sample, so the
realized recurrence range was compressed and the $\hat\chi$-trend was underpowered
(Spearman $+0.35$, $p=0.12$); a Fisher combination with an earlier seven-level
sweep reached $p=0.039$, but that combine assumed an independence the two sweeps
only partially have and sat one rounding step from failing. Rather than lean on it,
we re-ran the \emph{same} pre-specified twelve-level grid at full power, pointing
the loader at the full \texttt{lerobot/droid\_1.0.1} ($600$ distinct DROID scenes
vs.\ $100$): the only change is the dataset, not the design or the estimator.

The dose-response \textbf{replicates decisively}. $\hat\chi$ rises---rank-monotone,
with minor level-to-level wobble but no trend reversal---from $-5.2\times10^{-3}$ at
pure churn (dose $0$) to $+1.67\times10^{-2}$ at the recurrent end (dose $0.5$).
Across the twelve level means the trend is Spearman
$\rho=0.94$ (level-clustered permutation $p<10^{-4}$); the scale-free
rank-association trend is $\rho_s$-vs-dose $=0.97$ ($p<10^{-4}$); and every other
test agrees (Kendall $\tau=0.82$; OLS $R^2=0.83$; per-seed level-block $p<10^{-4}$;
low-vs-high-dose sign flip, split at $0.25$, Mann--Whitney $p=10^{-5}$). The
properly-powered replication thus supersedes both the underpowered single sweeps
and the fragile Fisher combine: the recurrence mechanism is confirmed, not merely
suggested. The full-power run cost \$1.12 on one L40S.

Two bounds remain. First, this is a dose-response in the value
\emph{proxy}-to-write coupling $\hat\chi$; whether the proxy itself tracks realized
task success is the separate, unvalidated question of \cref{sec:honesty}. Second,
the design and grid were pre-specified, but this full-DROID re-run was executed
after review (the natural fix to the disclosed loader limitation), so we report it
as a pre-specified-design replication at proper power, not as the original frozen
primary.

\begin{figure}[t]
  \centering
  \includegraphics[width=0.92\linewidth]{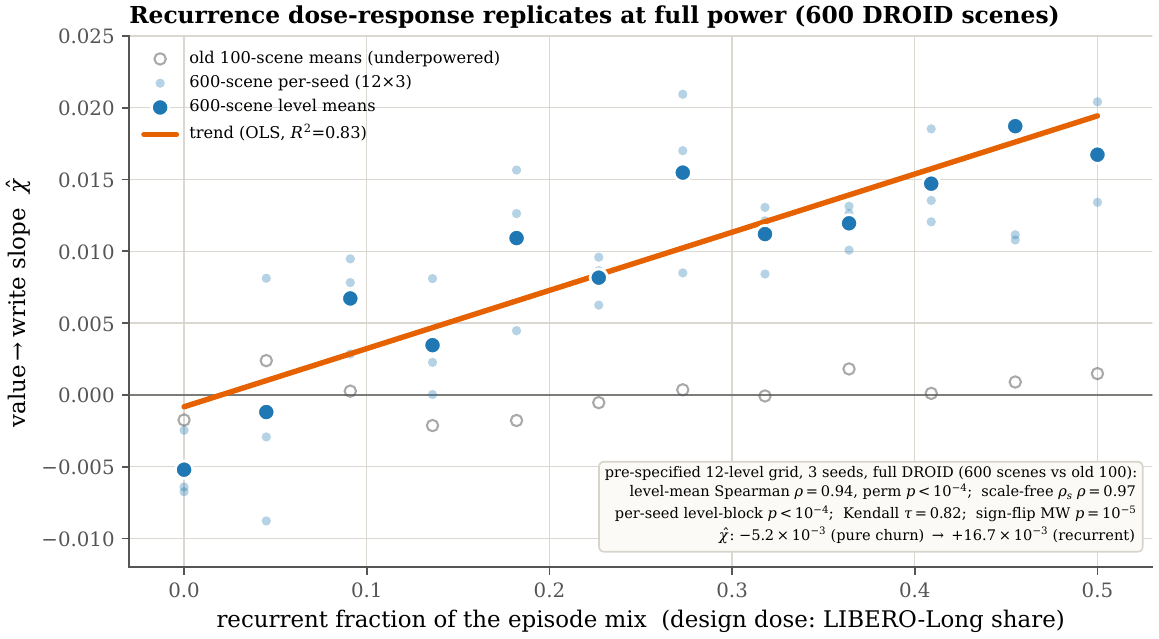}
  \caption{\textbf{Recurrence dose-response replicates at full power.} $\hat\chi$
  vs.\ the recurrent (LIBERO-Long) fraction of the episode mix, on the pre-specified
  twelve-level grid (three seeds) re-run with the full \texttt{droid\_1.0.1}
  ($600$ DROID scenes; light dots are seeds, filled circles level means, orange the
  OLS trend, $R^2=0.83$). The loader-capped $100$-scene means (hollow grey) cluster
  near zero; at full power $\hat\chi$ rises (rank-monotone) from $-5.2\times10^{-3}$
  (pure churn) to $+1.67\times10^{-2}$ (recurrent), Spearman $\rho=0.94$,
  permutation $p<10^{-4}$ across all trend tests. This supersedes the underpowered
  sweeps and the retired Fisher combine.}
  \label{fig:dose}
\end{figure}

\subsection{When the budget binds: commodity edge storage, not premium TLC}\label{sec:results-binding}
The dormancy we report---$\eta=0$ at datasheet prices---is a property of the
\emph{premium} storage we pinned to ($3{,}000$-P/E TLC on a $128$-GB module,
TBW$=128$~TB), not of the embodied-memory problem. That configuration gives a
device write-lifetime of $\approx5.2$ years at the measured fleet write demand
($24.4$~TB/robot/yr, \cref{sec:cscy-impl})---right at the edge of a $3$--$5$-year
deployment, so the budget only just fails to bind. But cheap edge robots run
\emph{denser, cheaper} NAND---commodity QLC and eMMC, whose endurance is
$\sim\!1{,}000$ P/E (a few hundred for the cheapest parts) rather than
$3{,}000$~\citep{purestorage_qlc,newegg2026ssdendurance}. There the same write
demand exhausts the stock within a deployment: a $128$-GB QLC part lasts
$\approx1.8$~years, a $64$-GB QLC part $\approx0.7$, a $32$-GB eMMC part
$\approx0.3$ (\cref{fig:binding}). \textbf{On the storage commodity edge robots
actually use, the endurance budget binds at datasheet prices ($\eta>0$) and the
wear-pricing layer is live.} The dormancy is a knife-edge artifact of the
premium-TLC pin, flipping to firmly-binding under exactly the cheaper-NAND regime
the regime map (\cref{fig:regimemap}) anticipates.

\paragraph{What binding buys: cost and lifetime, not task value.} We ran the
placement ladder in two binding regimes, and both \emph{tie} endurance-blind
routing on task value. (i) Under an artificial S2 cap (RAM-scarce, cloud expensive,
endurance capped at $0.4\times$ write demand; $5$ seeds $\times$ $5$ cells, binary
success proxy) the controller strictly beats the naive all-NVM AURA strategy
(McNemar $b=200$, $c=0$) but only ties trivial cloud-routing (H2 $p=1.0$): routing
to cloud avoids the endurance wall for free, so the binary proxy saturates. (ii)
Under the realistic commodity-QLC binding regime with a graded
net-realized-value metric, the wear-augmented index---clairvoyant optimum and
deployable $\eta$-routing alike---again ties endurance-blind routing (LRU,
size-based; advantage $0.00$, $95\%$ CI $[0,0]$ over scene-clustered resamples).
The reason is structural: realized value is \emph{tier-invariant}---RAM, NVM, and
cloud all serve the item---so the endurance rent reshapes \emph{costs} (erases,
energy, dollars, device lifetime), not task value, and a simple rule that routes
off flash captures the same value. A \emph{task-value} payoff would need both a
regime where flash is forced and scarce and a value signal validated against
realized task success (\cref{sec:honesty}); that is future work.

\begin{figure}[t]
  \centering
  \includegraphics[width=0.92\linewidth]{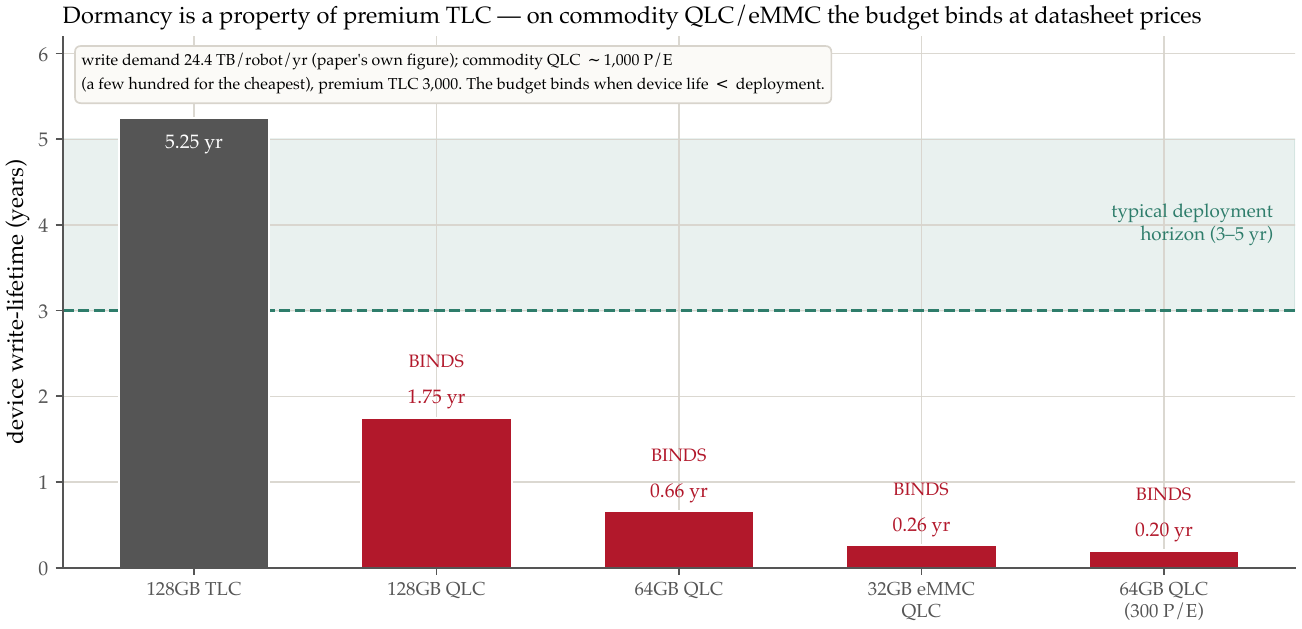}
  \caption{\textbf{The endurance budget binds on commodity storage.} Device
  write-lifetime at the measured fleet write demand ($24.4$~TB/robot/yr) across
  storage classes: premium $128$-GB TLC lasts $\approx5.2$~years (dormant, at the
  edge of deployment), but commodity QLC/eMMC ($\sim\!1{,}000$ P/E
  \citep{purestorage_qlc,newegg2026ssdendurance}) wears out in $0.2$--$1.8$~years, well inside a
  $3$--$5$-year deployment---so the budget binds at datasheet prices and the
  wear-pricing lever is live.}
  \label{fig:binding}
\end{figure}

\paragraph{Why the index ties: write-intensity has no dispersion to exploit.} The
tie is not an artifact of treating the cloud as a free escape hatch. Re-pricing the
cloud tier at the model's slow value $\Vslow$ (\cref{eq:PiRN}) and sweeping
connectivity from connected to fully disconnected leaves the wear-aware advantage at
$0.00$ in every regime. The cause is the value--write \emph{joint} distribution: on
LIBERO, write-intensity is nearly constant ($\mathrm{CV}(w)=0.13\%$,
$w\in[1.003,1.011]$) while value varies roughly two orders of magnitude more
($\mathrm{CV}(v)\approx24\%$ on the same items). The wear-augmented index ranks placement by surplus
\emph{per erase}; dividing by an almost-constant $w$ leaves the ranking unchanged,
so the index \emph{collapses to value-ranking} and the two policies coincide.
\Cref{fig:dispersion} makes the boundary quantitative: on a synthetic control with
$\chi>0$ held fixed and $\mathrm{CV}(w)$ swept, the advantage is flat at zero until
$\mathrm{CV}(w)\gtrsim 0.2$ and reaches $\approx 4\%$ at $\mathrm{CV}(w)=0.5$, while
LIBERO sits two orders of magnitude below, at the floor. Wear-aware placement thus
pays off precisely when high-value memories are \emph{disproportionately
rewritten}---a property of the workload, not the policy---which LIBERO-class
manipulation does not exhibit.

\begin{figure}[t]
  \centering
  \includegraphics[width=0.9\linewidth]{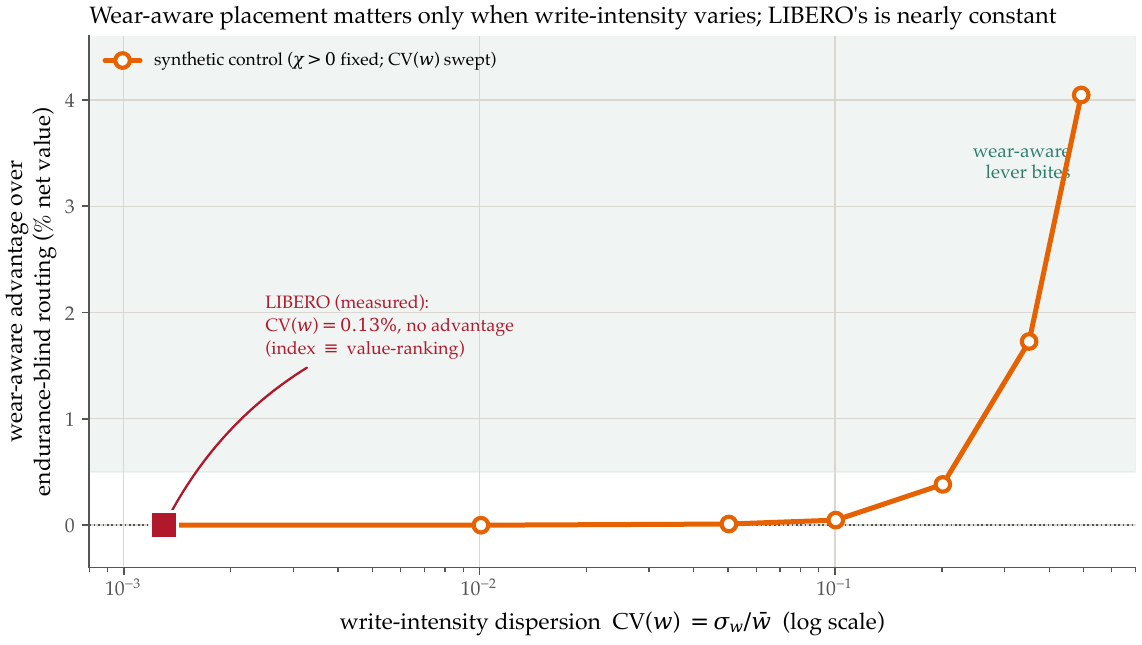}
  \caption{\textbf{When wear-aware placement beats endurance-blind routing.} The
  advantage in net realized (modeled) value is flat at zero until write-intensity
  dispersion $\mathrm{CV}(w)$ is substantial (synthetic control, $\chi>0$ fixed,
  $\mathrm{CV}(w)$ swept; cloud priced at the model's $\Vslow$ in the binding
  commodity-QLC regime). LIBERO's measured $\mathrm{CV}(w)=0.13\%$ sits at the
  floor, where the wear-augmented index collapses to value-ranking. Value is the
  proxy $\hat v$, not realized task success.}
  \label{fig:dispersion}
\end{figure}

\paragraph{The recurrence--dispersion tension: why the win regime is empty.} The
dispersion the lever needs ($\mathrm{CV}(w)\gtrsim 20\%$) is not merely unmet on
LIBERO---it is \emph{structurally} out of reach, because the two ingredients a
wear-aware win requires (a positive coupling $\chi>0$ \emph{and} high write-dispersion)
are anti-correlated across every way real robots generate writes (\cref{fig:tension}).
\textbf{Recurrence} (manipulation) makes $\chi$ positive---re-observing valuable
scenes couples value with writes---but re-observing the \emph{same} scenes homogenizes
the stream ($\mathrm{CV}(w)\le0.4\%$ on recurrent LIBERO and on a real SO-101 and a sim
ALOHA arm). \textbf{Churn} (teleoperation) spreads write-intensity but decouples it
from value, driving $\chi<0$. And \textbf{navigation}---the escape hatch we predicted,
where landmarks recur \emph{unequally}---does have the highest dispersion of any
workload ($\mathrm{CV}(w)$ of $7$--$10\%$ on the Berkeley GNM datasets), but its
frequently-traversed places are low-value transit while distinctive landmarks are seen
rarely, so value and writes \emph{anti}-correlate ($\chi$: $\rho_s=-0.13,-0.22$; proxy
validity $\rho\ge0.92$). We then surveyed write-dispersion across $\approx20$
Open-X / LeRobot workloads and ran the $\chi$ pipeline on the highest-dispersion
ones. Across thirteen workloads with measured $\chi$---spanning the ecosystem,
multiple embodiments, and all three write-generating mechanisms, plus the full-power
recurrence dose-response (\cref{sec:results-dose}; $\chi$ rises monotonically with
recurrence, $\rho=0.94$, which lifts $\chi$ exactly as it homogenizes the
writes)---the win quadrant ($\chi>0$ \emph{and} $\mathrm{CV}(w)>20\%$) is empty, and
empty by \emph{mechanism}, not by sampling. The single closest approach is a
bimanual xArm dataset ($\rho_s=+0.31$, $\hat\chi=+0.34$ at $\mathrm{CV}(w)=16\%$,
proxy $0.88$)---the lone positive coupling at high dispersion---but it is
underpowered ($n=70$ episodes, $\chi$ CI $[-0.11,+0.78]$ straddling zero) and still
short of the $20\%$ threshold, so we report it as a \emph{suggestive frontier}, not a
counterexample: coordination-rich manipulation is the regime a future positive
result should target. (All external embodiments use LeRobot v3.0, video-decoded
through the \emph{same} image-counterfactual $\chi$ pipeline with action/state
z-scored per dataset so $\mathrm{CV}(w)$ is comparable; held-out proxy validity
$\rho\ge0.76$ throughout.) A wear-aware placement win would require a workload where
high-value memories are frequently \emph{and} unequally rewritten \emph{with positive
coupling}; no natural regime we measured supplies all three, and we state this as the
precise, falsifiable boundary rather than engineer a workload to cross it.

\begin{figure}[t]
  \centering
  \includegraphics[width=0.92\linewidth]{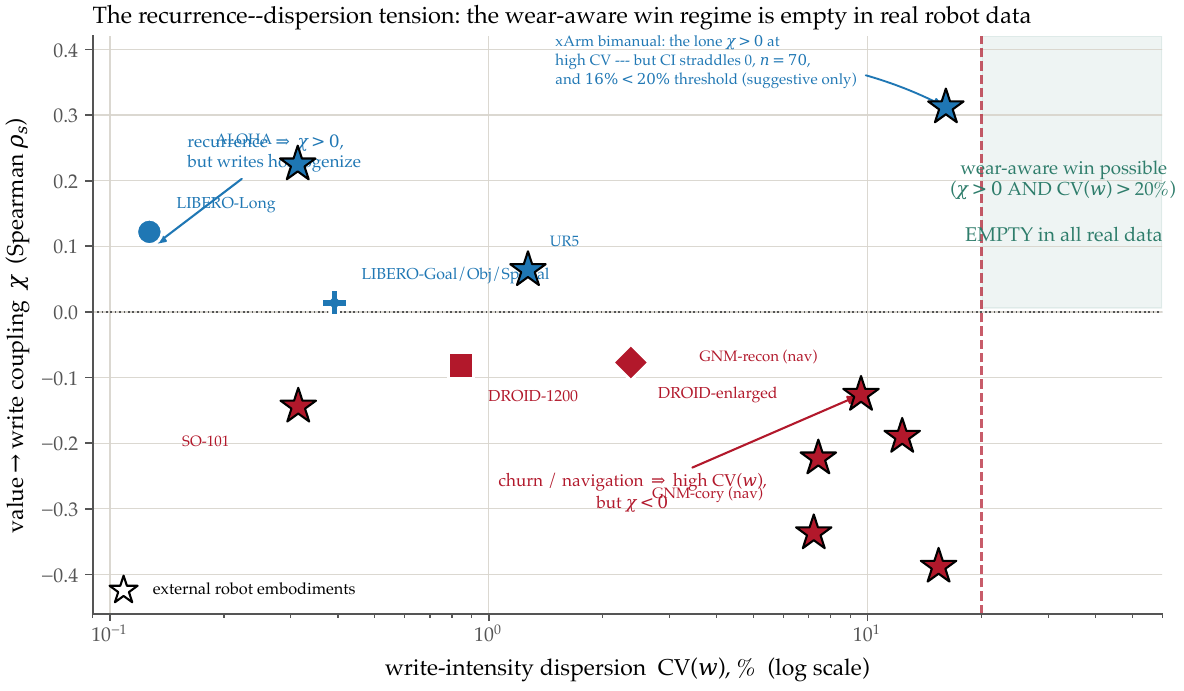}
  \caption{\textbf{The recurrence--dispersion tension.} A wear-aware win needs both
  $\chi>0$ and high write-dispersion $\mathrm{CV}(w)$, but they are anti-correlated
  across all three write-generating mechanisms: recurrence (manipulation) gives $\chi>0$
  at low dispersion; churn (teleoperation) and navigation give high dispersion but
  $\chi<0$. Four in-house workloads (round/plus/square/diamond) and nine external
  robot workloads (stars) from an $\approx20$-workload Open-X / LeRobot survey trace
  an anti-correlated band; the win quadrant ($\chi>0$, $\mathrm{CV}(w)>20\%$, threshold
  from \cref{fig:dispersion}) is empty. The lone positive coupling at high dispersion
  is a bimanual xArm dataset ($\rho_s{=}+0.31$ at $\mathrm{CV}(w){=}16\%$), flagged
  \emph{suggestive} because its $\chi$ CI straddles zero ($n{=}70$) and it remains
  below the $20\%$ threshold---a frontier for future work, not a counterexample. The
  full-power recurrence dose-response is reported separately (\cref{sec:results-dose}).}
  \label{fig:tension}
\end{figure}

\subsection{Placement controller: a negative result (H1, H2, H3)}\label{sec:results-controller}
The $3.15$M-parameter controller (BC warm-start $+$ PPO, $5$ seeds) returns a
negative result. \textbf{H1} (non-monotone deny-NVM-to-high-value persistence
slope) is not rejected on any seed ($p\ge0.40$; the slope is flat at $0.0$).
\textbf{H2} (beat the strongest cost-matched baseline) ties on every seed
($p=1.0$). \textbf{H3} (controller $\ne$ AURA single-tier) is seed-dependent, and
where it rejects it is an \emph{energy-metric-gaming artifact}: two seeds reject
AURA by emitting MIGRATE actions that deflate the joule denominator, driving a
controller/oracle ratio to $2.11$---impossible against a clairvoyant oracle
(\cref{fig:p2ctrl}). The one non-gaming seed converges to AURA-identical behavior
(McNemar $b=c=0$). The verdict is invariant across all four stream-stitch families
and every price regime, confirming it is a fixed policy property, not a wear
effect; PPO does not beat the BC warm-start. The central cause: at datasheet
S0-connected prices the endurance budget \emph{never binds} (the solved dual is
zero with $196$-KB LIBERO frames), making H1/H3 partly vacuous---itself a boundary
result, which the binding-regime tests (\cref{sec:results-binding}) address.

\begin{figure}[t]
  \centering
  \includegraphics[width=0.92\linewidth]{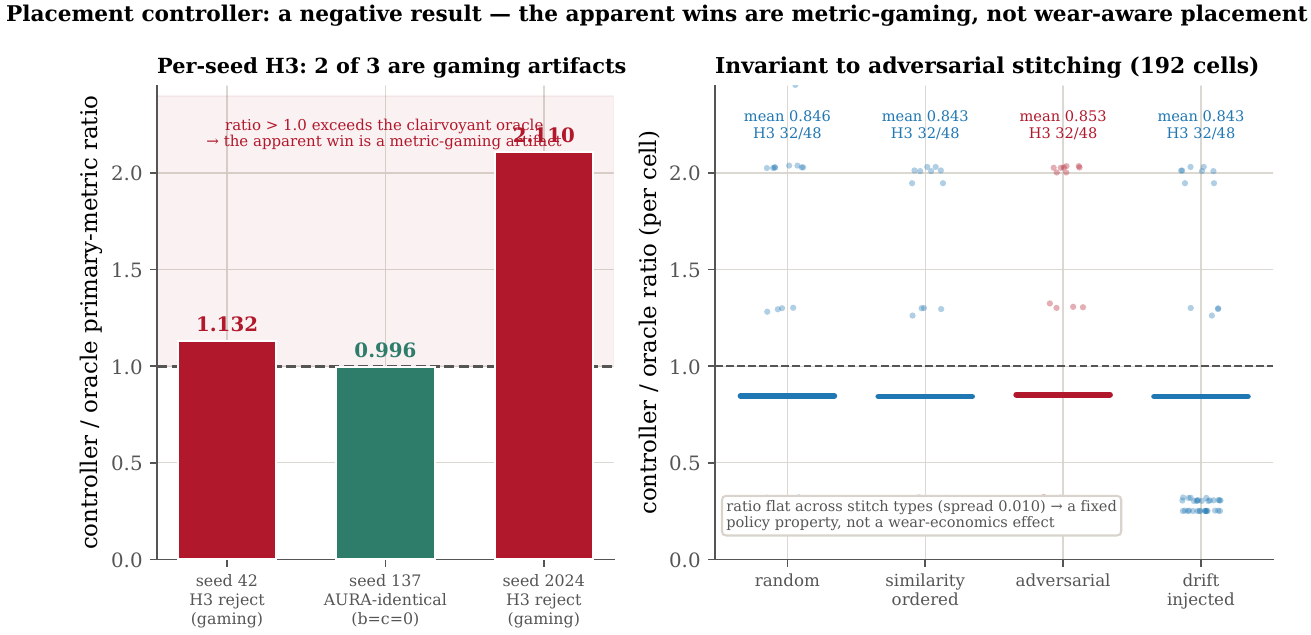}
  \caption{\textbf{Controller robustness and the metric-gaming finding.} Left:
  per-seed H3 outcome (controller/oracle primary-metric ratio); the non-gaming seed
  is AURA-identical (ratio $0.996$), the rejecting seeds exceed the oracle
  (ratios $1.13$, $2.11$, the impossibility signature of gaming). Right: the H3
  verdict is invariant across all four stitch families (spread $0.010$;
  adversarial $\equiv$ random), confirming the difference is a fixed policy
  property, not a wear-economics effect.}
  \label{fig:p2ctrl}
\end{figure}

\subsection{Price statics and the cross-partial (P5a--d, H5, H6)}\label{sec:results-price}
The placement simulator on the real LIBERO-labeled items with the fitted
$\hat\delta=0.032$ and datasheet constants, swept across the S0/S1/S2 price grid
returns \textbf{three of four pre-specified
signs confirmed and one inconclusive null}. P5a ($p_{\mathrm{NAND}}\uparrow\Rightarrow
\eta\downarrow$), P5c ($p_{\mathrm{egress}}\uparrow\Rightarrow\eta\uparrow$), and
P5d ($\Eend\uparrow\Rightarrow\eta\downarrow$) confirm with tight bootstrap CIs
and $100\%$ sign-stability. \textbf{P5b} ($p_{\mathrm{DDR}}\uparrow\Rightarrow
\eta\uparrow$) is an \emph{inconclusive null} ($+7.2\times10^{-8}\approx0$, CI
$[-4.6\times10^{-7},+4.6\times10^{-7}]$, $25\%$ sign-stable): under binding RAM the
shadow price $\mu_R$ absorbs the DDR shock one-for-one, so it never reaches the
endurance margin, a clean identification of the boundary condition, exactly as
rescoped in \cref{prop:statics}. The NVM-share $\sigma_N$ partials are
inconclusive by construction ($\sigma_N$ is budget-pinned at $\approx E_{\mathrm{end}}/\bar w$
once $\eta>0$, so the price signal lives in $\eta$). The \textbf{cross-partial}
(H6) point estimate is $+0.50$ ($87\%$ sign-stable) but its bootstrap CI
$[-0.34,+1.25]$ straddles zero, so it is \emph{directionally consistent,
CI-inconclusive}: H6's non-fatal kill is not triggered, but the sign is not
CI-confirmed either.

\paragraph{Quantitative headline.} The equilibrium rent declines monotonically
across the price band ($\etasim$: LOW $3.16\to$ S0 $2.42\to$ S1 $1.90\to$ S2
$1.47$, all $\times10^{-4}$, simulator units; see \cref{sec:calibration} for why
these are not on the $\etamark$ scale); the break-even item value is $\vbe=0.91$
at base prices. \textbf{The 2025--26 NAND supercycle (S0$\to$S2) cuts $\etasim$ by
$\approx39\%$} ($2.42\to1.47\times10^{-4}$) while leaving $\sigma_N$ ($+0.7$ pp)
and the break-even $\vbe=0.91$ unchanged: the price shock is absorbed by the
wear margin, not the placement boundary (\cref{fig:price}).

\paragraph{Not yet run.} The phase-diagram-measure and RAM-pressure hypotheses
H4/H4b (\cref{prop:ramslack}) require a dedicated swept-box / $C_R$-sweep run that
has not landed; we report them as outstanding rather than fill them from an
unrelated artifact.

\begin{figure}[t]
  \centering
  \includegraphics[width=0.86\linewidth]{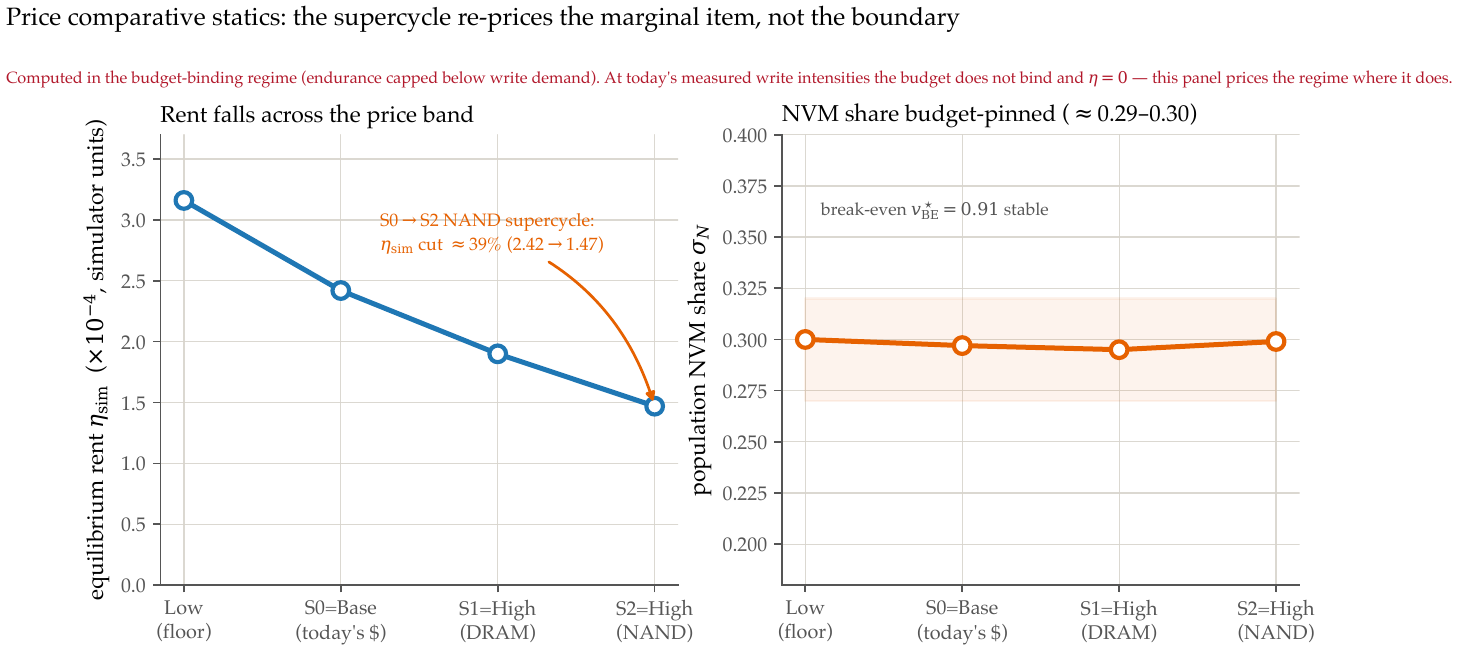}
  \caption{Price comparative-statics fan over the unified Low/Base/High band
  (real re-solved outputs): the equilibrium rent $\etasim$ declines monotonically
  across the band (left) while the population NVM share $\sigma_N$ stays
  budget-pinned at $\approx0.29$--$0.30$ (right), with the break-even
  $\vbe=0.91$ annotated.}
  \label{fig:price}
\end{figure}

\section{Discussion}\label{sec:discussion}
\paragraph{What the priced model buys.}
The central object is the endurance shadow price $\eta$: it fixes the
persist/evict boundary, signs how placement reacts to the memory-price supercycle
(\cref{sec:pricestatics}), and doubles as a device-lifetime price. The learned
controller adds little on top: in both binding regimes (\cref{sec:results-binding})
it ties price-based routing on task value, because realized value is tier-invariant
across RAM/NVM/cloud---so once $\eta$ and the wear-augmented index are in hand,
simple price-based routing suffices on today's hardware. The genuinely open
question is therefore narrower than ``does the controller help'': whether a regime
exists where flash is \emph{forced} and scarce (so the tier choice is not free)
\emph{and} a value signal validated against realized task success
(\cref{sec:honesty}) makes placement causally move performance. That is the next
experiment.

\begin{figure}[t]
  \centering
  \includegraphics[width=0.92\linewidth]{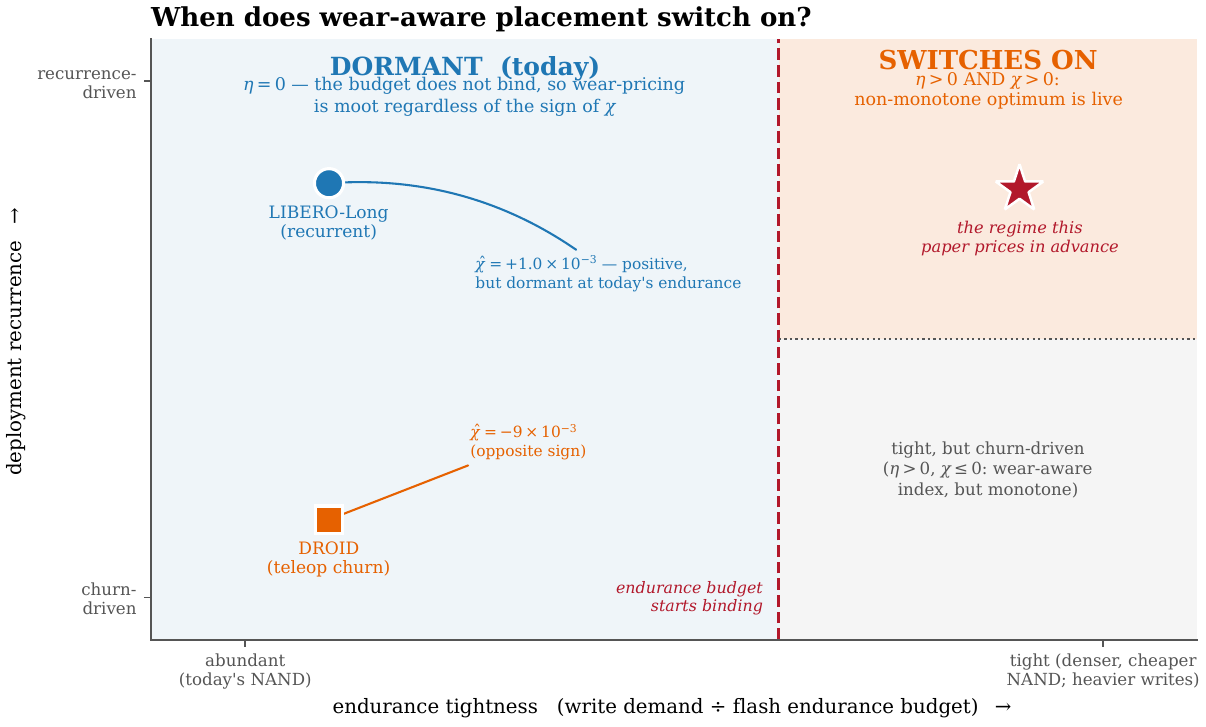}
  \caption{\textbf{When wear-aware placement switches on.} The lever needs
  \emph{both} a binding endurance budget ($\eta>0$, the right of the map) and a
  positive value--write coupling ($\chi>0$, the top): only the upper-right cell
  activates the non-monotone optimum. Our measured datasets sit in the
  endurance-abundant band as run (LIBERO-Long recurrent $\chi>0$; DROID churn-driven
  $\chi<0$), but the binding (right) column is \emph{live today} on the commodity
  QLC/eMMC cheaper edge robots use (\cref{sec:results-binding}), not merely future.
  What remains priced \emph{in advance} is the upper-right corner alone---binding
  endurance \emph{and} a recurrent, $\chi>0$ regime together.}
  \label{fig:regimemap}
\end{figure}

\subsection{Limitations: what this work does \emph{not} establish}\label{sec:honesty}
\begin{enumerate}[leftmargin=1.4em,itemsep=2pt]
\item \textbf{The association $\chi>0$ is regime- and backbone-conditional}, not
universal: positive only for SmolVLA-0.5B on recurrent long-horizon data, null on
a second LIBERO suite, negative on teleoperation (\cref{sec:results-chi}). The
non-monotone branch (\cref{prop:nonmono}) is claimed only in that regime; the
monotone index and rent $\eta$ are sign-agnostic and hold in every measured cell.
\item \textbf{The DROID negative is post-hoc, not pre-specified}: the enlargement
followed a negative, underpowered pilot. We label it exploratory, report it on the
pre-specified $1{,}200$-new-only subset, and treat it as a regime-difference signal
pending replication.
\item \textbf{Cross-backbone comparison is uninterpretable by a pre-specified
criterion}: OpenVLA-7B and SmolVLA agree on item value at only $\rho_s=0.05$ (floor
$0.6$), so we do not read the 7B arm as a disconfirmation. A cross-backbone
agreement floor is a necessary validity check before any cross-model sign claim.
\item \textbf{A pre-specified third backbone (pi0-3.5B) was loaded but deferred}
after a checkpoint state-dict mismatch left its vision tower random-initialized,
which would invalidate its value proxy; we report the two-backbone scale axis
(SmolVLA-0.5B, OpenVLA-7B) instead, with details in \cref{app:repro}.
\item \textbf{The controller win is not demonstrated; one H3 rejection was a
metric-gaming artifact} (a seed exceeding a clairvoyant oracle by deflating the
energy metric; \cref{sec:results-controller}). At datasheet prices the budget never
binds, so the offline H1/H3 tests are partly vacuous.
\item \textbf{The placement$\rightarrow$task-success causal chain is not yet
demonstrated with a VLA in the loop.} We built and ran the project's first true
VLA-in-the-loop arm: a LIBERO-finetuned SmolVLA-0.5B acting in real
LIBERO-Long physics (not labeling static frames): $18$ episodes, $3{,}960$
environment steps, $0$ errors. The
pre-specified causal-room gate required oracle-memory minus no-memory
$\ge +8$~pp and measured $+0.0$~pp: the $0.5$B backbone solves none of the
three hardest LIBERO-Long tasks at a $220$-step budget, so there is no success
signal for placement to perturb, and the minimal training-free memory channel
does not reproduce the published end-to-end-trained memory effect. We
\emph{aborted} the full campaign per the frozen kill criterion rather than
engineering a favorable coupling, and we explicitly scope
placement$\rightarrow$task-success causal validation, which would require a
trained memory-augmented backbone, to future work. A direct consequence bears on
the headline: every $\chi$ we report is a coupling between write-intensity and a
counterfactual value \emph{proxy} $\hat v$, internally validated against the full
SmolVLA counterfactual ($r_{\mathrm{cf}}=0.92$) but \emph{not} validated
end-to-end against realized task success, since the causal gate found no
task-success signal to anchor it. The measured regime-dependence of
$\sign\chi$ stands as a property of $\hat v$; tying it to realized task value
awaits a backbone that can solve the tasks.
\item \textbf{Welfare}: we minimize the operator's private expected cost: no
claim on social welfare or Pareto-efficiency.
\item \textbf{Market equilibrium}: $\mathbf p$ is an exogenous shadow/contract
price; \cref{sec:pricestatics} is partial-equilibrium, operator-side.
\item \textbf{Closed-form $\eta$ path}: the $(1+r)^t$ path is not claimed; $\eta$
is a learned dual under non-stationary demand.
\item \textbf{Separability failure modes}: the index/threshold form is exact only
under item separability; under retrieval complements, write-amplification,
non-stationary $F_t$, and lumpy items, the \emph{direction} of non-monotonicity
survives but the closed-form threshold does not: the case for a learned
controller, theory as scaffold.
\end{enumerate}

\subsection{cs.CY implications}\label{sec:cscy-impl}
By \cref{cor:lifetime}, cost-optimal forgetting is also lifetime-extending: a
fleet-wide erase saving $q$ defers replacement embodied carbon at the $\approx
22$~kg~CO$_2$e/TB anchor~\citep{weppe2025nandcarbon} (which prices NAND
\emph{manufactured}, not bytes \emph{written}). For a $1{,}000$-robot fleet the
solver puts device write-lifetime at $\approx5.2$ years, mapping to $\approx24.4$
TB of replacement NAND per year (fleet capacity $\div$ write-lifetime, \emph{not}
TB-written) and fleet embodied carbon of order $\approx540$~kg~CO$_2$e/yr; the
supercycle shortens device life $\approx2.4\%$, second-order relative to the
$39\%$ swing in $\etasim$. These are order-of-magnitude, relative figures, not
forecasts. Lifetime extension \emph{complements} rather than substitutes for
fleet-refresh economics~\citep{switzer2022junkyard,balyo2025robotfleettco} and
circular-economy policy~\citep{eu2024rtr,eu2024espr}; edge data-residency rules add
a privacy rationale for local persistence~\citep{edge2025gdpr}.

\paragraph{Author continuity.}
AURA~\citep{chen2026aura} is the front-end write gate and a baseline arm;
\citet{chen2026decodegap} supplies edge-decode cost anchors;
\citet{chen2026aegis} is the statistics-methodology precedent.

\section{Conclusion}\label{sec:conclusion}
Under a binding non-renewable write-endurance budget, the cost-minimizing memory
placement rule is a threshold in a wear-augmented per-byte index governed by an
endurance shadow price $\eta$, a rule whose form does not depend on the sign
of the value--write association. On this sign-agnostic spine, the optimum becomes
\emph{strictly non-monotone} in item value under one further condition, a positive
association $\chi>0$ that we measure at a pre-specified gate rather than assume.
Our central empirical finding is that the antecedent's sign is a property of the
deployment regime: positive on recurrent long-horizon data with a small backbone,
null on a second suite, negative on non-recurrent teleoperation, and
uninterpretable across backbones that fail a cross-backbone agreement floor.
The measured boundary of relevance matters as much, and it cuts both ways: on
premium $3{,}000$-P/E TLC the endurance budget does not bind, the confirmed
coupling is small, the calibrated down-crossing lies outside the measured value
support, and our learned controller does not yet beat a simple price-based routing
rule under a binary success proxy---but on the commodity QLC/eMMC
($\sim\!1{,}000$ P/E) that cheaper edge robots actually run, the same measured
write demand exhausts the endurance stock within a deployment, so the budget binds
at datasheet prices and the pricing layer is \emph{economically live}
(\cref{sec:results-binding}). Yet even there, the wear-aware policy \emph{ties}
simple endurance-blind routing on task value: realized value is tier-invariant
across RAM/NVM/cloud, so the rent reshapes cost and device lifetime, not
performance---and the placement gain stays zero because the wear-aware index has no
write-dispersion to exploit. That last point is structural: the recurrence that
makes $\chi$ positive homogenizes the write stream, so $\chi>0$ and high dispersion
are anti-correlated and the win regime sits empty in every workload we measure
(\cref{fig:tension}). Re-solving the calibrated model across the 2025--26 memory-price
supercycle cuts the equilibrium rent by $\approx39\%$ while leaving the
persist/evict boundary fixed, and the rent doubles as a device-lifetime price.
Wear-aware placement is thus economically live today on commodity edge storage
(\cref{fig:regimemap}), but its \emph{task-value} payoff awaits a regime where flash
is forced and scarce and a value signal validated against realized task success.
That is the next step.

\appendix
\section{Proofs}\label{app:proofs}

\begin{proof}[Proof of \cref{prop:monotone}]
With $\eta=0$ the three \emph{persistent} returns share an identical slope in $v$.
Writing $V_i=\lambda_i v_i/(1-\gamma e^{-\delta_i})$, both $\Vfast_i$ (entering
$\Pi^R,\Pi^N$) and $\Vslow_i=\Vfast_i-\lambda_i\ell\pi/(1-\gamma e^{-\delta_i})$
(entering $\Pi^C$) are affine in $v_i$ with the \emph{same} coefficient
$\lambda_i/(1-\gamma e^{-\delta_i})$; the tier-specific terms ($\mu_R{+}p_R$,
$p_N s_i$, $p_C s_i$, and at $\eta=0$ the bounded cash wear $\cwear w_i$) are
$v$-independent intercepts. Hence $\Pi^R,\Pi^N,\Pi^C$ are \emph{parallel} lines in
$v$: the choice \emph{among} the persistent tiers is fixed by their intercepts and
does not vary with $v$: there is no unique ``largest $v$-coefficient'' tier.
The \emph{only} $v$-dependent margin is persist-vs-discard: $\Pi^\varnothing
=-\Pr[\text{needed}]\kappa_i$ has the smallest $v$-slope (zero, or sub-linear by
\cref{ass:recompute}), so $\max(\Pi^R,\Pi^N,\Pi^C)-\Pi^\varnothing$ is strictly
increasing in $v$. Single-crossing of this margin gives a threshold $v^\dagger$
above which some persistent tier dominates discard. Above $v^\dagger$, when the
$v$-independent intercepts make NVM the best persistent tier (i.e.\
$\Pi^N\ge\max(\Pi^R,\Pi^C)$, the condition $e_N<e_C+\ell\pi$ on the energy/latency
terms), $\Ind[x_i=N]$ is weakly increasing in $v_i$, a step up driven by the
persist-vs-discard margin, not by any tier owning the steepest $v$-slope.
\end{proof}

\begin{proof}[Proof of \cref{prop:nonmono}]
Define the slack $g(v):=B(v)-(\cwear+\eta)\wbar(v)$.
\emph{(i) Low $v$:} $V$ small, even locality below the discard/cloud option:
$\Pr[x=N\mid v]$ low, rising as $V(v)$ clears the cloud margin.
\emph{(ii) Moderate $v$:} $g(v)>0$ and $V(v)$ beats the cloud option: persist.
\emph{(iii) High $v$:} $B'(v)\approx 0$ (fact 1) while $(\cwear+\eta)\chi>0$
(fact 2), so $g'(v)<0$; for $\eta>\bar\eta$ the slack crosses zero from above at
interior $\vdc$, after which \eqref{eq:benc} fails and the item routes to
cloud (or cheap recompute, \cref{ass:recompute}). Hence $\Pr[x=N\mid v]$ strictly
falls past $\vdc$: rise-then-fall. By the implicit function theorem on
$g(\vdc;\eta,\chi)=0$ with
$g_v=B'(\vdc)-(\cwear+\eta)\chi<0$:
$\partial\vdc/\partial\eta=-g_\eta/g_v$, where $g_\eta=-\wbar(\vdc)<0$, so
$\partial\vdc/\partial\eta=-\wbar(\vdc)/[(\cwear+\eta)\chi-B'(\vdc)]<0$.
Writing $\wbar(v)=w_0+\chi v$ locally, $g_\chi=-(\cwear+\eta)\vdc<0$, giving
$\partial\vdc/\partial\chi=-g_\chi/g_v<0$. The $\chi>0$ requirement is
necessary: at $\chi=0$, $g'(v)=B'(v)\approx 0$ and no interior down-crossing
exists (monotone, \cref{prop:monotone}).
\end{proof}

\begin{proof}[Proof of \cref{prop:ramslack}]
The down-crossing is fixed by \eqref{eq:benc}, the NVM-vs-(cloud/recompute)
margin, whose RHS rises in $v$ via $\chi>0$ \emph{independent of} $\mu_R$.
$\mu_R$ enters only \eqref{eq:benr}, i.e.\ \emph{which} non-NVM tier wins. Both
$\Pi^C=\Vslow-p_C s$ and $\Pi^\varnothing=-\Pr[\text{needed}]\kappa$ are
$\mu_R$-free and available (recompute cheap for high $v$ by
\cref{ass:recompute}); so $\Pr[x=N\mid v]$ falls past $\vdc$ for any
$\mu_R\ge 0$, with $\arg\max(\Pi^R,\Pi^C,\Pi^\varnothing)$ re-labeling the
receiver. The location result follows from the budget identity (raising $\mu_R$
raises NVM demand, hence $\eta$) and $\partial\vdc/\partial\eta<0$.
\end{proof}

\begin{proof}[Proof of \cref{prop:statics} (sketch)]
\begin{sloppypar}
Each sign differentiates the break-even conditions \eqref{eq:benc}/\eqref{eq:benr}
and the budget identity $\sum\sum w=\Eend$ implicitly defining $\eta(\mathbf p)$.
P5a/P5d are clean (single-signed channels); P5b/P5c carry a directional label
where the second-order $\eta$-feedback opposes the first-order channel, as noted.
\end{sloppypar}
\end{proof}

\section{Artifact and Reproducibility}\label{app:repro}
\sloppy
All figures and statistics are regenerable from the released artifact bundle
(\texttt{wamp\_reproducibility/}, distributed as supplementary material with this
paper): the \texttt{code/} tree (simulator, eval harness, cost model,
$\chi$ estimators, price-statics and calibration scripts, tests), the frozen
pre-specified plan (\texttt{experiments/experiment\_plan.md}), the canonical
$\chi$ re-analysis table
(\texttt{experiments/chi\_reanalysis/chi\_canonical\_table.json}, the source of
truth for every $\chi$ reported here), every per-phase analysis output, and the run
and cost registries (\texttt{runs/}, $46$ billable rows summing to \$18.3764). A
\texttt{MANIFEST.md} at the artifact root maps each paper claim to its regeneration
script and data file. The decision log (D-001--D-016) and the multi-round
governance audit trail are referenced by identifier throughout and are included in
the complete repository release alongside this code/data core. The
only exclusions are the raw PPO training-log directories ($\approx$146\,MB,
regenerable from the pinned seeds and configs) and the external datasets and
checkpoints, which are not redistributed but are pinned by HuggingFace slug and
revision in the dataset card. Total project compute spend was $\approx\$18.38$
(the original pre-specified campaign spent $\approx\$17.26$---its kill criteria,
not the budget, stopped it---and the subsequent full-power replication and
commodity-storage binding analyses added $\approx\$1.12$).

\paragraph{Software versions (pinned).}
Throughout this appendix, internal campaign codes are used for provenance: W1/W2
denote the first and second $\chi$ batteries (W2 is the dose-response sweep,
``W2 top-up'' its pre-specified replication), W3a the closed-loop diagnostic
folded into \cref{sec:results-binding}, and D-\textit{nnn} entries reference the shipped
decision log. The campaign used \emph{three} coexisting Python stacks (D-003): a
main train/eval env (\texttt{transformers} 4.55.4) for the Phase-2 controller, W3a
diagnostic, P3 statics, and calibration; the Modal $\chi$-estimation containers
(\texttt{transformers} 4.51.3 with \texttt{lerobot[smolvla]} 0.3.3) for every
SmolVLA arm (Phase-0 gate, labeling, DROID re-test, all $\chi$ batteries, and the
deferred pi0-3.5B probe); and an isolated OpenVLA env (\texttt{transformers} 4.40.1
$+$ \texttt{timm} 0.9.10) for the 7B arm. The headline and regime $\chi$ results
depend on the 4.51.3 stack and the OpenVLA arm on 4.40.1; no $\chi$ estimate
depends on the main env. Full version pins ship in \texttt{code/pyproject.toml} and
the artifact.

\paragraph{Random seeds.}
Controller training/eval seeds Modal $\{42,137,2024\}$ and Lambda $\{7,99\}$;
bootstrap seeds $\{42,137,2024,7,99,2718\}$ (primary $=137$); W2 top-up seeds
$\{1001,1002,1003\}$. All $\chi$ CIs use $1{,}000$-resample physical-scene-clustered
bootstraps.

\paragraph{Hardware and per-phase cost.}
Modal (L40S, T4) and a single Lambda H100-SXM lane (terminated and verified).
Key per-phase spend from \texttt{runs/cost\_registry.csv}:

\begin{table}[h]
\centering\small\setlength{\tabcolsep}{5pt}
\begin{tabular}{llr}
\toprule
Phase / arm & Hardware & Cost (USD)\\
\midrule
Infra \& staging smoke & Modal L40S & 0.0230\\
Phase-0 full gate & Modal L40S & 0.2748\\
Phase-1 full labeling & Modal L40S & 0.9736\\
Phase-2 placement controller ($\ge3$ seeds) & Modal L40S/T4 & 2.4542\\
DROID re-test (enlarged) & Modal L40S & 0.5540\\
Lambda lane (seeds 7/99 $+$ OpenVLA arm) & Lambda H100-SXM & 7.2600\\
$\chi$ re-analysis (OpenVLA relabel) & Modal L40S & 0.8500\\
Phase-3 price comparative-statics & CPU (local) & 0.0000\\
W1+W2 $\chi$ suites $+$ dose-response & Modal L40S & 1.1931\\
pi0-3.5B probes (deferred, D-013) & Lambda H100-SXM & 0.0570\\
W3a closed-loop diagnostic & Modal L40S & 0.0800\\
W2 top-up dose-response & Modal L40S & 0.8004\\
VLA-loop smoke (HV0/HV1, parallel) & Modal L40S & 2.7381\\
\midrule
Original campaign subtotal &  & 17.2582\\
Full-power 600-scene replication (\cref{sec:results-dose}) & Modal L40S & 1.1182\\
Commodity-storage binding + ladder analyses (\cref{sec:results-binding}) & CPU (local) & 0.0000\\
\midrule
\textbf{Cumulative project total} &  & \textbf{18.3764}\\
\bottomrule
\end{tabular}
\end{table}

\noindent The table is exhaustive over the $46$ billable rows of
\texttt{runs/\allowbreak cost\_registry.csv} (REFERENCE-only rows excluded); rows
sum exactly to the \$18.3764 cumulative total. The original pre-specified campaign
spent \$17.2582 (its kill criteria, not the budget, stopped it); the subsequent
full-power replication and commodity-storage binding analyses added \$1.12.

\paragraph{Data and checkpoint slugs (HuggingFace).}
Datasets: \texttt{lerobot/droid\_1.0.1} (full DROID; analyzed 1{,}200-new-only
subset), \texttt{lerobot/droid\_100} (rev
\texttt{87301a2}), and the LIBERO suites via
\texttt{openvla/\allowbreak modified\_libero\_rlds} (rev \texttt{6ce6aaa}).
Checkpoints: \texttt{lerobot/smolvla\_base},
\texttt{HuggingFaceVLA/\allowbreak smolvla\_libero},
\texttt{moojink/\allowbreak openvla-7b-oft-\allowbreak finetuned-libero-10}, and
\texttt{lerobot/pi0} (deferred). Revisions are pinned in the dataset card where
determinable from \texttt{experiments/\allowbreak p1\_staging\_manifest.json};
otherwise stated as \emph{latest as of 2026-06-12}. Full per-arm episode/scene
counts, preprocessing, and license notes are in the dataset card pointer:
\texttt{experiments/DATASET\_CARD.md}.

\paragraph{Controller hyperparameters.}
Placement controller (\texttt{code/configs/phase2\_controller.yaml}): a $3.15$M-parameter
transformer ($d_{\mathrm{model}}{=}256$, $8$ heads, $4$ layers, FF mult $4$, $5$
actions \{keep-RAM, write-NVM, offload-cloud, discard, migrate\}, max $512$ items;
$\le 50$M cap). Training: behavior-cloning warm-start ($5$ epochs, lr $10^{-3}$,
hindsight-oracle targets) then clipped-surrogate PPO (rollout length $2048$,
lr $3\times10^{-4}$, $\gamma{=}0.95$, GAE $\lambda{=}0.95$, clip $\epsilon{=}0.2$,
$4$ PPO epochs, minibatch $256$, $200{,}000$ total steps), $\ge 3$ seeds; scene
split $70/15/15$; reward $=$ success-proxy $-\lambda_E\,$energy $-\nu\,$erase
$-\rho_C\,$cloud-\$$\,-\,$migration-cost $0.05$.

\paragraph{Script $\rightarrow$ claim map.}
Headline $\hat\chi$ + matrix: \texttt{phase0\_gate.py}, \texttt{w1w2\_chi.py}
($\to$ \texttt{chi\_\allowbreak canonical\_table.json}); OpenVLA scale-stress arm:
\texttt{openvla\_\allowbreak confirmation\_arm.py}; DROID enlarged (post-hoc):
\texttt{droid\_retest.py}; W2 dose-response $+$ top-up: \texttt{w2\_topup.py},
\texttt{power\_\allowbreak analysis\{,2,3\}.py}; controller (negative result):
\texttt{phase2\_train.py}/\texttt{full.py}, \texttt{w3a\_rollout.py}
($\to$ \texttt{GATE\_\allowbreak VERDICT.json}); $\chi$ canonical re-analysis:
\texttt{chi\_estimators.py} (a verbatim port of
\texttt{phase0\_gate.\allowbreak \_local\_linear\_slope}).

\paragraph{Run-ID provenance.}
The run identifiers behind each result (relocated here from inline tags for the
camera-ready build) are listed in \cref{tab:runids}.

\begin{table}[h]
\centering\small\setlength{\tabcolsep}{5pt}
\caption{Run-ID provenance for the main results.}
\label{tab:runids}
\begin{tabular}{ll}
\toprule
Result & Run ID(s)\\
\midrule
Phase-0 value--write gate & \texttt{phase0-full-20260612} (GATE-cov)\\
Canonical $\chi$ matrix / re-analysis & \texttt{chi-reanalysis-20260612} (D-012); \texttt{W1+W2 battery}\\
Phase-1 labeling / regime mechanism & \texttt{phase1-full-20260612}; \texttt{droid-retest-20260612}\\
Recurrence dose-response & \texttt{W2 battery}; \texttt{W2-topup}\\
Controller (H1/H2/H3) & \texttt{P2 eval battery} ($64$ cells/seed, Modal/Lambda lanes)\\
Closed-loop W3a & \texttt{w3a-diagnostic-5seed-20260612}\\
Price statics / cross-partial & \texttt{p3-pricestat-20260612}\\
Calibration ($\delta$-anchor) & \texttt{econ-calibration-20260612}; \texttt{droid-retest-20260612}\\
VLA-loop (aborted) & \texttt{vla-loop-smoke-parallel-20260612}\\
\bottomrule
\end{tabular}
\end{table}

\bibliographystyle{plainnat}
\bibliography{references}

\end{document}